\documentclass{amsart}
\usepackage[dvipsnames]{xcolor}
\usepackage{relsize}
\usepackage{amssymb,hyperref}
\usepackage{mathtools}
\usepackage{cleveref}
\usepackage{fullpage}
\usepackage{bm}

\usepackage{tikz} 
\usetikzlibrary{cd,arrows,shapes,shapes.multipart, intersections,decorations,decorations.markings,decorations.pathreplacing,fit}
\tikzset{thick/.style={line width=.6mm}}
	\tikzstyle{hugetensor}=[rounded rectangle,thick,draw=black,minimum width=40mm,minimum height = 3mm]
	\tikzstyle{squaretensor}=[rounded rectangle,thick,draw=black,minimum width=15mm,minimum height = 7mm]
	\tikzstyle{littletensor}=[circle,thick,draw=black,fill=blue!30,minimum size=.5mm]
	\tikzstyle{tinytensor}=[circle,thick,draw=black,fill=blue!30,inner sep=0pt,minimum size=6pt]

        \newcommand{\Path}{\operatorname{Path}}
        \newcommand{\len}{\operatorname{len}}
        
        \newcommand{\Max}{\textit{Max}}
        \newcommand{\Min}{\textit{Min}}
        \newcommand{\actions}{A}
                \newcommand{\states}{S}
\renewcommand{\leq}{\leqslant}
\renewcommand{\geq}{\geqslant}

\newcommand{\cont}{\textrm{co}}
\renewcommand{\stop}{\textrm{st}}
\theoremstyle{plain}
\newtheorem*{theorem*}{Theorem}
\newtheorem{proposition}{Proposition}
\newtheorem{theorem}{Theorem}

\theoremstyle{definition}
\newtheorem{definition}{Definition}
\newtheorem*{notation*}{Notation}

\newtheorem{remark}{Remark}

\DeclareMathOperator{\sgn}{sgn}

\newcommand{\R}{\mathbb{R}}

\usepackage{pgf}
\usepackage{tikz}
\usepackage{tikz-cd}
\usetikzlibrary{cd}
\usetikzlibrary{calc}
\usetikzlibrary{arrows}
\usetikzlibrary{shapes}

\title[Relu and softplus neural nets as zero-sum turn-based games]{Relu and softplus neural nets  as zero-sum turn-based games}

\usepackage[disable,colorinlistoftodos,bordercolor=orange,backgroundcolor=orange!20,linecolor=orange,textsize=scriptsize]{todonotes}

\author{St\'ephane Gaubert}
\author{Yiannis Vlassopoulos}

\address{SG: INRIA and CMAP, \'Ecole polytechnique, IP Paris, CNRS}
\email{Stephane.Gaubert@inria.fr}
\address{YV: ATHENA Research Center, Athens, Greece and IHES, Bures-sur-Yvette, France}
\email{yvlassop@gmail.com, yvlassop@ihes.fr}

\date{\today}

\begin{document}

\begin{abstract}
We show that the output of a ReLU neural network can be interpreted as the value of a zero-sum, turn-based, stopping game, which we call the ReLU net game. The game runs in the direction opposite to that of the network, and the input of the network serves as the terminal reward of the game. 
In fact, evaluating the network is the same as running the Shapley-Bellman backward recursion for the value of the game.
Using the expression of the value of the game as an expected total payoff with respect to the path measure induced by the transition probabilities and a pair of optimal policies, we derive a discrete Feynman-Kac-type path-integral formula for the network output. 
This game-theoretic representation can be used to derive bounds on the output
from bounds on the input,
leveraging the monotonicity of Shapley operators,
and to verify robustness properties  using policies as certificates.
Moreover, training the neural network becomes an inverse game problem: given pairs of terminal rewards and corresponding values, one seeks transition probabilities and rewards of a game that reproduces them. Finally, we show that a similar approach applies to neural networks with Softplus activation functions, where the ReLU net game is replaced by its entropic regularization.
\end{abstract}
\maketitle
{\small\tableofcontents}
\section{Introduction}
\subsection{Summary of results}
Neural networks are the engine behind the artificial intelligence revolution;
however they are for the most part treated as black boxes with data going in and out after a huge number of simple operations such as additions, multiplications and taking maxima.

This poses a fundamental challenge in the quest for control
over the possible behaviours of a given neural net. 
In this paper we take a step towards
elucidating the mathematical structure underlying ReLU neural networks. 

To that end, we prove that \textit{the map computed by a ReLU neural network coincides with the value of a two-player, turn-based, zero-sum, stopping  game}, thought of as a function of the terminal payoff.
In fact evaluating the neural network from input to output turns out to be the same thing as running the Shapley-Bellman backward recursion for the value of this game (Theorem~\ref{theorem 1}). We call this game the \textit{ReLU net game}.

Furthermore, \textit{the map computed by a Softplus neural network  coincides with the value of the ReLU net game when the latter is entropically regularized by adding the logarithm of the policy probability to its payoff} (\Cref{theorem 2}).
Analogously, evaluating the Softplus neural net is the same as running the Shapley-Bellman backward recursion for the entropically regularized game. We refer to this entropically regularized ReLU net game as the \textit{Softplus net game}. 

Using the fact that the value of the game is achieved by a pair of
optimal policies, we obtain
a Feynman-Kac-type representation of the network output as a discrete path integral, see~\eqref{e-path2} and~\eqref{Vop}.

Our constructions allow us to interpret a supervised learning problem
as an {\em inverse game problem}, in which a collection of terminal payoffs
and state values are known, and the parameters of the game (instantaneous rewards and transition probabilities) have to be inferred.

We also point out two applications of the game-theoretic interpretation.
The game representation of a ReLU neural
network provides an \emph{order-preserving lift} of the function computed
by the neural network. This monotonicity property
allows one to propagate lower and upper bounds on the input
to corresponding bounds on the output (\Cref{prop-bound}). This
approach can also be applied
to the verification of neural network properties. In particular,
we show that the policies
of the two players provide certificates that the input satisfies
a given property, or its negation, see~\Cref{prop-poly}.
In the simplest case, when the network is used as a single-output classifier
that accepts or rejects inputs based on thresholds applied to the output, Player \Max\ aims to certify that the property holds, whereas
Player \Min\ aims to the certify that the opposite property holds.

We now briefly describe the ReLU net game and our main theorems. 
\subsection{The ReLU net game}
First, recall the map computed by a ReLU net. We have $\operatorname{ReLU}(x)\coloneqq  \max(x,0)$.
Assume the network has $L$ layers and layer $l$ has $k_l$ neurons. 
\textit{We number the layers starting from the output of the neural network (layer $1$) to its input (layer $L$)}.
The weight matrix of layer $l$ is denoted by $W^l$ and there are bias vectors $b^l \in \mathbb{R}^{k_l}$ in each layer $l$. 
The input vector is $x\in \mathbb{R}^{k_L}$.

Then the total output function of the net is $f:\mathbb{R}^{k_L} \to \mathbb{R}^{k_1}$ where 
\begin{equation}
\label{NNo}
 f(x)=\max(W^1\max(W^{2}(\max \dots (\max(W^{L-1}(\max(W^Lx+b^L,0))+b^{L-1},0)\dots)) +b^{2},0)+b^1,0)
\end{equation}
and the $\max$  is applied coordinate-wise on a vector. 

The ReLU net game is played in the opposite direction to that the neural net is running, so that the value
given by the Shapley-Bellman backward recursion is computed in the direction of the neural net. 

There are two players,  \textit{Max} who aims to maximize the reward and \textit{Min} who aims to minimize it.
Every layer of the neural net corresponds to a stage of the game.
The game starts at the end of the neural net and proceeds towards its beginning.
Every neuron $(l,i)$ -- the $i$th neuron in the $l$th layer, counting from the end of the neural net -- gives rise to two game states, $(l,i,+)$ in which the maximizer plays and $(l,i,-)$ in which the minimizer plays. The possible actions at any state are two: \emph{stop} or \emph{continue}, where to stop means going to an absorbing, so-called ``cemetery'' state with zero instantaneous reward.  The weights $W^l_{i,j}$ of the ReLU net are used to define state transition probabilities: $P^l_{i+,j+}=P^l_{i-,j-}\coloneqq(\gamma^l_i)^{-1}{(W^l_{i,j})^+}$
and
$P^l_{i+,j-}=P^l_{i-,j+}\coloneqq(\gamma^l_i)^{-1}{(W^l_{i,j})^-}$ where
$\gamma^l_i\coloneqq\sum_j|W^l_{i,j}|$ is a discount factor (allowed to take values which may exceed one),
and $(\cdot)^\pm$ denotes the positive or negative part of a real number.
Thus, if the weight between two neurons is positive the same player (whether \emph{Max} or \emph{Min}) keeps playing and if it is negative, the player changes. 
If the  bias at a neuron $(l,i)$ is $b^l_i$
then the reward at state $(l,i,+)$ is $b^l_i$  and at $(l,i,-)$ it is $-b^l_i$. 
The terminal reward of the game at the states $(L,j,\pm)$ 
is given by $\pm x_j$ where $x_j$ is the $j$th input of the neural net.

A deterministic policy $\bm{\pi}$ for Player \textit{Max} is an assignment of either $0$ (signifying the stopping action) or $1$ (for the continue action) to all the states $(l,i,+)$. Analogously we can consider a deterministic policy $\bm{\sigma}$ for Player $\textit{Min}$. Given a pair of policies,
we can define a probability distribution on game trajectories. 
Indeed, note that since the only possible actions are stop or continue, a game trajectory is a sequence of states. The probability of such a trajectory 
is defined to be the product of all state transition probabilities along the trajectory. Moreover along such a trajectory, the rewards are accumulated. 
We denote by $V^{l,\bm{\pi},\bm{\sigma}}_{i+}(x)$ the expected value of accumulated rewards over all game trajectories starting at $(l,i+)$, when the terminal reward at state $(L,\pm j)$ is given by $\pm x_j$.
Then the value of the game at state $(l,i+)$ is the saddle-point (Nash equilibrium) value
$$V^l_{i+}(x)\coloneqq
\max_{\bm{\pi}}\min_{\bm{\sigma}}V^{l,\bm{\pi},\bm{\sigma}}_{i+}(x)=\min_{\bm{\sigma}}\max_{\bm{\pi}}V^{l,\bm{\pi},\bm{\sigma}}_{i+}(x).$$
Let $y^l_i$ be the output of the $i$th neuron in the $l$th layer of the neural net when the input is $x$.
We then prove in Theorem 1 (Section 4) that:
\begin{equation}
    y^l_i=V^l_{i+}(x)=-V^l_{i-}(x)
\end{equation}
In particular for the output of the net we have 
$f(x)_i=y^1_i=V^1_{i+}(x)=-V^1_{i-}(x)$.
Furthermore, the value is achieved by the optimal policies $\bm{\pi}^*(x)$ and $\bm{\sigma}^*(x)$ which are given by 
\begin{equation}
 \bm{\pi}^*(x)((l,i,+))=
\begin{cases}
  1, & \text{if } \sum_j W^l_{i,j}y^{l+1}_j +b^l_i \geq 0, \\
  0  & \text{if }\sum_j W^l_{i,j}y^{l+1}_j+b^l_i \leq 0 .
\end{cases}
\end{equation}
and 
 $\bm{\sigma}^*(x)((l,i,-))=1-\bm{\pi}^*(x)((l,i,+))$.

Since the value of the game is 
the piecewise-linear map computed by the neural net,
we see that the pair of optimal policies for a given terminal reward (i.e.\
input to the neural net) defines a linearity region of the map computed by the neural net. More precisely, a linearity region is a region where the pair of optimal policies of the two players does not change as we consider the value as a function of final reward. When the optimal policy changes we enter a new linearity region of the value i.e.\ the map computed by the neural net.  

These descriptions are summarized in Table \ref{tab:NN-Games1}.

\begin{table}[h]
  \centering
  \caption{\textbf{ReLU Neural net / Turn based zero-sum stopping Game, correspondence }}
  \label{tab:NN-Games1}
  \footnotesize
  \begin{tabular}
  {|l|l|}
    \hline
    \textbf{ReLU Neural net} & \textbf{ReLU net game} \\
    \hline
    layer $l$, counting from the end of the neural net& stage $l$ of game \\ \hline
   $(l,i)$: neuron $i$ in $l$th layer & states $(l,i+)$ and $(l,i-)$ 
   \\ \hline
    & Cemetery state (absorbing and with 0 reward)
   
   \\
   \hline
    $L$ : depth of net & horizon of game \\ \hline
    $b^l_i$: Bias in $(l,i)$ & reward $r^l_{i,+}\coloneqq b^l_i$ in $(l,i,+)$ and $r^l_{i,-}\coloneqq -b^l_i$ in $(l,i,-)$  \\ \hline
$\gamma^l_i\coloneqq\sum_j|W^l_{i,j}|$ & Discount factor
\\ \hline
 &Actions at any state: stop or continue (stop means go to cemetery state) \\ \hline 
    $W^l_{i,j}$: Weight
    & Transition Probabilities $P^l_{i+,j+}=P^l_{i-,j-}\coloneqq (\gamma^l_i)^{-1}{(W^l_{i,j})^+}$\\
&\text{ and }
$P^l_{i+,j-}=P^l_{i-,j+}\coloneqq(\gamma^l_i)^{-1}{(W^l_{i,j})^-}$\\ \hline
    Input $x=(x_1, \dots, x_n)$ to the net & $\pm$ terminal reward of the game, at states $(L,1,\pm),\dots (L,n,\pm)$\\ \hline
    $y_i^l$: the output of neuron $(l,i)$ & $V^l_{i,+}(x)$: Value at game state $(l,i,+)$\\ \hline
    $-y_i^l$ & $V^l_{i,-}(x)$:  Value at game state $(l,i,-)$\\  \hline
    Output of net & Value $V^1_{i,+}(x)$ of initial state of game  \\ \hline
    Linearity Region of map computed by NN & Pair of policies that are optimal for a terminal reward\\ \hline
    Training & Inverse game problem 
     \\ \hline

  \end{tabular}
\end{table}

Recall that stochastic two-player games generalize Markov decision processes, which can be thought of as games with a single player. The Shapley-Bellman equation in that case reduces to the MDP Bellman equation. In fact \textit{in the case where all the weights of the neural net are positive, the game described above collapses to a MDP and the output of the neural net is the value of that MDP}. 

\subsection{The Softplus net game}
By applying an entropic regularization to the game, namely adding the log of the policy probability to the payoff at every state, we obtain a game whose Shapley-Bellman backward recursion is the same as the outcome of a Softplus neural net. 
\textit{The value of the game can then be interpreted as a free energy and the optimal policy is the Gibbs measure.} 

This construction generalizes the previous one since we recover the ReLU case by sending the temperature parameter that appears with the entropic regularization to zero.
We explain this in Section 8. This is illustrated in~\Cref{tab:NN-Games2}.

\begin{table}[ht]
 \centering
  \caption{\textbf{Softplus Neural net / Entropy-regularized zero-sum turn based stopping Game, correspondence }}
  \label{tab:NN-Games2}
\footnotesize%
  \begin{tabular}
  {|l|l|}
    \hline
    \textbf{Softplus Neural nets} & \textbf{Entropy-regularized Games}  \\
    \hline
    layer $l$,  counting from the end of the neural net & stage $l$ of the game \\  \hline
    $L$ : depth of net& horizon of the game
    \\ \hline
    neuron $(l,i)$& game states $(l,i,+)$ and $(l,i,-)$ \\
    \hline
     & Cemetery state (absorbing and with 0 reward)
    
\\ \hline
 Softplus parameter $\tau$ & entropy parameter $\tau$ \\  \hline
    $b^l_i$: Bias at neuron $(l,i)$ & \begin{tabular}{l} reward given  policy $\bm{\pi}$, $R^l(l,i+,\cont)\coloneqq b^l_i-\tau \log(\pi^l(\cont|i+))$,\\
$R^l(l,i+,\stop)\coloneqq-\tau \log(\pi^l(\stop|i+))$\\
$R^l(l,i-,\cont)\coloneqq-b^l_i+\tau \log(\pi^l(\cont|i+)),
R^l(l,i-,\stop)\coloneqq\tau \log(\pi^l(\stop|i+))$\end{tabular}
     \\ \hline
$\gamma^l_i\coloneqq\sum_j|W^l_{i,j}|$ & Discount factor
     \\  \hline
     $W^l_{i,j}$: Weight, 
                                  & Transition Probabilities $P^l_{i+,j+}=P^l_{i-,j-}\coloneqq
                                    {(W^l_{i,j})^+}$\\
&\text{ and }
$P^l_{i+,j-}=P^l_{i-,j+}\coloneqq(\gamma^l_i)^{-1}{(W^l_{i,j})^-}$  
     \\ \hline 
     Input $x=(x_1,\dots x_n)$ to the net & Terminal reward of the game  \\ \hline
    
$y^l_{i,\tau}$: output of neuron $(l,i)$  & Value $V^l_{i+,\tau}(x)$ at game state $(l,i,+)$  
\\  \hline
$-y^l_{i,\tau}$ & Value $V^l_{i-,\tau}(x)$ at game state $(l,i,-)$  \\ \hline  
Output of the net & Value $V^1_{i+,\tau}(x)$ of the initial state $(1,i,+)$ of the game
    \\ \hline
    Training & Inverse game problem 
    \\ \hline
  \end{tabular}
\end{table}

\subsection{Related work}
Game-theoretic representations of functions have appeared in different contexts. Evans established in~\cite[Lemma~4.2]{evans} a general minimax representation theorem for Lipschitz Hamiltonians. In that way, abstract Hamilton-Jacobi PDE can be effectively interpreted as dynamic programming equations of differential games. The infinity Laplacian is a remarkable example
of a concrete stochastic game (``tug of war'') hidden behind a nonlinear PDE,
as shown by Peres, Schramm, Sheffield, and Wilson~\cite{Peres2008}.
Kohn and Serfaty showed that the PDE governing
mean curvature motion can be represented by a deterministic game~\cite{Kohn2005}. In the discrete time setting, Kolokoltsov used Evans' result to establish a game representation of order preserving sup-norm nonexpansive mappings, see~\cite{kolokoltsov,maslovkololtsov95}.
A general minimax representation theorem for nonexpansive mappings, with a game interpretation, appeared in~\cite{1605.04518}.
This should also be compared with a result of Ovchinnikov~\cite{ovchinnikov},
showing that a continuous piecewise-linear functionts admits
a finitely described minimax representation.

Another series of works interprets neural networks in terms
of tropical or ``piecewise-linear'' geometry.
Zhang, Naitzat, and Lim~\cite{zhang2018} showed that the map realized by a neural network
can be written as the difference of two tropical polynomial functions (with real exponents). This is further elaborated in~\cite{Maragos1, Maragos2, Maragos3,Kordonis2025}. This differs from the nested minimax representation implied by~\Cref{theorem 1}. The specifity of the representation in~\Cref{theorem 1} is its
{\em monotone character}, with respect to the coordinate-wise order,
  will all weights in~\eqref{e-Val plus}-\eqref{e-Val minus} nonnegative.
Results from polyhedral geometry (on the number of vertices of Minkowski sums)
allow one to bound the number of
linearity regions of maps defined by ReLU neural networks, see in particular
Zhang, Naitzat and Lim, {\em op.\ cit.}, and Mont\'{u}far, Ren and Zhang~\cite{montufar}.
Linearity regions have a natural interpretation in terms
of pairs of policies in the associated ReLU game. 

The verification of neural networks, discussed in~\Cref{sec-certificates},
is currently a topic of intense interest~\cite{Huangsurvey2020}. Efficient techniques rely
on (nonconvex) mathematical programming methods; see~\cite{crown}
and the references therein. We also note that the idea of {\em policy iteration}
has been applied to program verification, see~\cite{adjegaubertgoubault10}.

As said above, our results allow us to interpret the training of a ReLU neural network
as an inverse game problem; this is a two-player version of the inverse optimal control problem studied in various contexts, see e.g.~\cite{levine}.
\subsection{Organization of the paper}
Sections~2 and 3 are introductions to Markov decision processes (MDP) and zero-sum games respectively, explaining the basic structures and the Shapley-Bellman backward recursion equations which serve as our main tool.
Section~4 explains in detail the ReLU net / ReLU game correspondence.
Note that a MDP can be viewed as a single player game; it corresponds to a ReLU net with only positive weights.
Section~5 provides some applications of this correspondence.
Sections~6 and 7 are the analogues of Sections~2 and 3 where we add entropic regularization. Finally Section 8 explains that adding entropic regularization to the ReLU net game gives the Softplus net game which reproduces Softplus nets.

\subsection{Acknowledgments} YV would like to thank Michael Douglas and Maxim Kontsevich for useful conversations. He also would like to thank IHES for providing excellent working conditions. 

\section{Basics of Markov decision processes}
We first  start by recalling the special case of a one player stochastic game, namely a Markov decision process (MDP). The materials of this section
are standard, we refer the reader to~\cite{whittle86,puterman2014markov,Lasserre} for background. We explain what is a MDP and how to compute its value together with optimal policies. We will be explicit in this simpler case in order to establish the basic equations.
The 2-player extension will be the subject of the next section.  
Note however that, as explained below, \textit{even the simpler MDP case actually corresponds to the special case of a ReLU neural net all of whose weights are positive}.
\begin{definition}\label{def-MDP}
A MDP with finite horizon $T$ is a tuple $(S,A,P,r,\gamma,\phi,T)$ where 
$S$ is a finite set of states,
$A=(A(s))_{s\in S}$ is a family such that $A(s)$ is the set of possible actions in state $s$, $P_t(s'|s,a)$ is the probability distribution for moving to any other state $s'$ when starting at state $s$ and taking action $a$ at stage $t$,
$r_t(s,a,s') \in \mathbb{R}$ is the reward for taking action $a$ at the same stage, while at state $s$ and arriving at stage $s'$,
$\gamma_t(s)>0$ is the discount factor in state $s$ at stage $t$,
$T$ is the number of stages of the process (the horizon) and $\phi(s)$ is the terminal reward at time $T$ .
\end{definition}
For every choice of $(t,s,a)$, we require that $\sum_{s'}P_t(s'|s,a)=1$.

For simplicity of exposition, we assume here that that for all $s\in S$,
the set of actions $A(s)$ is finite. However, the results which follow carry over, with straightforward changes,
to the case in which $A(s)$ is a separable, metrizable and compact topological space, with the reward and the distribution probability depending continuously on the action $a\in A(s)$, see e.g.~\cite{Lasserre} (we will need such a general setting when considering entropically regularized MDP).
We define a \textit{randomized policy} %
to be a map $\pi:S \to \cup_{s\in S} \Delta(A(s))$,
where $\Delta(A(s))$ denotes the simplex generated by the action space $A(s)$,
i.e., the set of probability measures over $A(s)$.
It gives the probability of different possible actions available at a given state $s \in S$. 
If the probability distribution is supported on a single action, it is called a deterministic policy.

Given a sequence of (randomized) policies $\bm{\pi}=(\pi_1,\dots,\pi_{T-1})$, 
we define a \textit{probability measure on state-action paths}, namely on MDP trajectories of the form
\begin{equation} \label{MDP path}
\alpha\coloneqq(s_t,a_t,s_{t+1},a_{t+1},\dots s_{T-1},a_{T-1},s_T), \qquad s_t,\dots,s_T \in S, \; a_k \in A(s_k) \text{ for }t\leq k\leq T-1\enspace ,
\end{equation}
by
\begin{equation}
\label{prob MDP path}
P^{\bm{\pi}}(\alpha|s_t)\coloneqq\prod_{i=t}^{T-1}\pi_i(a_i|s_i)P_t(s_{i+1}|s_i,a_i)\enspace.
\end{equation}
The discounted reward accumulated along the path $\alpha$ is
defined by:
    \begin{equation}
        r(\alpha)\coloneqq r_t(s_t,a_t)+\gamma_t(s_t)r_{t+1}(s_{t+1},a_{t+1})+\dots +
        \big(\prod_{k=t}^{T-2}\gamma_k(s_k))r_{T-1}(s_{T-1},a_{T-1})+\big(\prod_{k=t}^{T-1}\gamma_k(s_k)) \phi(s_T)\enspace .
        \label{eq-r(alpha)}
      \end{equation}
      We do not require here that $\gamma_t\leq 1$ as is customary in applications of MDP to mathematical economy. This will be necessary since we will have to scale the weights of the neural net in order to produce probabilities.
We then
define $V^{t,\bm{\pi}}_s$, the stage $t$ \textit{value function} evaluated at state $s$ under the sequence of policies
$\bm{\pi}$, to be the expected value of the sum of discounted rewards along all paths from state $s$ till the end of horizon:
\begin{equation}
  \label{MDP value given policy}
      V^{t,\bm{\pi}}_s\coloneqq E_s^{\bm{\pi}} r(\alpha)
      =\sum_\alpha P^{\bm{\pi}}(\alpha|s_t)r(\alpha)  
      \enspace ,
    \end{equation}
    so that $E_s^{\bm{\pi}}$ denotes the expectation with respect to the probability measure induced by the sequence of policies $\bm{\pi}$, over the set
    of paths $\alpha$ of the form~\eqref{MDP path} with initial state $s_t=s$.
    In particular, if $t=T$ is the terminal time, we have
    \begin{equation}
    \label{MDP terminal val}
    V^{T,\bm{\pi}}_s=\phi(s) \enspace,
    \end{equation}
the final reward $\phi$ playing the role of a boundary condition.
We see from~\eqref{eq-r(alpha)} and~\eqref{MDP value given policy} that $V^{\bm{\pi}}$ is a discrete path integral.
  We now have  the backward Kolmogorov recursion.
 \begin{equation}
 \label{MDP back Kolm}
V^{t,\bm{\pi}}_s
= \sum_{a\in A(s)} \pi_t(a|s)
  \Bigl[
    r(s,a)
    + \sum_{s'\in {S}} \gamma_t(s) P_t(s'\mid s,a)\,V^{t+1,\bm{\pi}}_{s'}
  \Bigr].
 \end{equation}
 See e.g.~\cite[Chap~4.2]{Norris1997MarkovChains} for background.

This backward recursion is the discrete analogue of the backward Kolmogorov equation and we  see that the solution as expected from the Feynman-Kac formula, is a discrete path integral given by the expectation value in Eqn~\eqref{MDP value given policy}.
The recursion in Eqn~\eqref{MDP back Kolm} describes the change of value as we make one step back (from $t+1$ to $t$) in the MDP states. The path integral arises when we iterate the recursion to arrive to the maximal number of steps from stage $t$ all the way to the horizon of the MDP.
To compute a solution we start from the final reward in Eqn~\eqref{MDP terminal val} and we proceed backward.

Moreover we define the \textit{value of the MDP} at stage $t$ with initial state $s$ to be  
\begin{equation}
V^t_s\coloneqq\max_{\bm{\pi}} V^{t,\bm{\pi}}_s,
\end{equation}
where the maximum is taken over sequences of policies.
We then have for $V^t_s$
  the basic equation of 
 dynamic programming, the Bellman equation, which is also a backward recursion.
    The start of the recursion is again the final reward given by Eqn~\eqref{MDP terminal val}. 
Then, the Bellman equation is
\begin{equation}
\label{MDP Bellman 2}
  V^t_s=
\max_{\pi\in \Delta(A(s))} \sum_{a\in\actions(s)}  \pi(a|s)
  \Bigl[
    r(s,a)
    + \sum_{s'\in\states} \gamma_t(s)P(s'\mid s,a)\,V^{t+1}_{s'}
  \Bigr].
\end{equation}
Note that Eqn~\eqref{MDP Bellman 2} is equivalent to 
\begin{equation}
\label{MDP Bellman}
V^t_s
= \max_{a\in\actions(s)}
\Bigl[
    r(s,a)
    + \sum_{s'\in\states}   \gamma_t(s) P(s'\mid s,a)\,V^{t+1}_{s'}
  \Bigr],
\end{equation}
where the maximum is taken over the set of actions,
since the function we are maximizing is linear and continuous in $\pi$,
and therefore the $\max$ over the simplex $\Delta(s)$ will occur at extreme points, which are precisely Dirac measures (Theorem~15.9 in~\cite{Aliprantis}). Therefore instead of taking the max over distributions $\pi$ over actions, we can take the max over actions. 

Equation \eqref{MDP Bellman} expresses the Bellman optimality principle, namely that if a policy is optimal for the whole horizon then it will also be optimal for the part of the MDP from any stage till the end. 

Note also that for the finite horizon problem considered here, we can construct step by step the optimal policy, whose existence therefore is guaranteed.
An optimal policy at time $t$ satisfies
\begin{equation}
\label{MDP opt policy}
    \pi^*_t(s)\in \arg\max_{a\in A(s)}
      \Bigl[r(s,a)
    + \sum_{s'\in\states} \gamma_t(s) P_t(s'\mid s,a)\,V^{t+1}_{s'}\Bigr] \enspace .
\end{equation}
Equation \eqref{MDP Bellman}, for the value $V^t_s$ is the discrete version of the Hamilton-Jacobi-Bellman equation from control theory. 
Given a sequence of policies
$\bm{\pi}$, we define a nonstationary Markov chain
on the set of states $\states$, as well as the expected instaneous payoff,
\begin{equation}
\label{MDP state prob}
P_t^{\bm{\pi}}(s'|s)\coloneqq\sum_a\pi_t(a|s)P_t(s'|s,a),\qquad
r(s)\coloneqq\sum_a\pi_t(a|s)r(s,a)\enspace .
\end{equation}
The recursion Eqn~\eqref{MDP back Kolm} can then be written in condensed form as:
\begin{equation}
V^{t,\bm{\pi}}_s
=r(s)
+ \sum_{s'\in\states} \gamma_t(s) P_t^{\bm{\pi}}(s'\mid s)\,V^{t+1,\bm{\pi}}_{s'} \enspace .
\end{equation}

There is a particular kind of MDP called a \textit{stopping MDP} which has only two possible actions at any state: ``Stop or Continue''. We realize the stopping action by adding a so-called \textit{cemetery state} $\bot$ which is absorbing and has zero reward. Absorbing means that once a player arrives there then they stay there forever. We will use its zero-sum game generalization in our interpretation of ReLU nets in Section 4.

\section{Basics of Repeated zero-sum Games}
We now explain the basics of the theory of repeated zero-sum games which will be used in the next section in order to describe the game representing a ReLU neural net.
The main novelty by comparison with the MDP case considered above is that, instead of the optimal value corresponding to a $\max$ over policies, it will be a saddle-point value, namely a max over policies of one player and a min over policies of the other player. This is a special case of Nash equilibrium value.
\subsection{Concurrent games}
We start with a simple extension (with state-dependent discount factor and finite horizon) of the model originally introduced by Shapley~\cite{shapley_stochastic}, see~\cite{sorin_repeated_games,solan} for recent presentations.
\begin{definition}[Finite-Horizon Zero–Sum Game]
A finite-horizon zero–sum game is defined by the tuple
\[
(S,\; A^1,\;A^2,\;P,\;r,\;\gamma,\;\phi,\; T),
\]
where $S$ is a finite set of states, $T$ is the horizon, $\gamma$ a discount factor, $\phi$ the terminal reward, as per~\Cref{def-MDP},
and
\begin{itemize}
  \item[--] $A^1=(A^1(s))_{s \in S}$, $A^2=(A^2(s))_{s\in S}$ are families
    such that for all states $s\in S$, $A^1(s)$ and $A^2(s)$ are the sets of possible actions for Player 1 (\textit{Max}, the maximizer) and Player 2 (\textit{Min}, the minimizer) in state $s$, respectively,
  \item[--] $P_t(s'\mid s,a^1,a^2)$ is the probability of transitioning to state $s'$ if the pair of actions $(a^1,a^2)$ is chosen in state $s$ at stage $t$ by the two players;
  \item[--] $r_t(s,a^1,a^2)$ is the stage-$t$ reward of Player \textit{Max}, and $-r_t(s,a^1,a^2)$ is the reward of player \textit{Min}, under the same circumstances.
\end{itemize}
\end{definition}
Again, we assume for simplicity that every set $A^i(s)$ with $i\in\{1,2\}$ and $s\in S$ is finite, referring the reader to~\cite{sorin_repeated_games} for the case of compact action spaces.
We assume that at any time $t$, both players are informed of the current state. In the original model of Shapley (often called {\em concurrent games}, as opposed to the more special turn-based games defined in the next section),
they both play simultaneously (neither one waits to see the other move).

A randomized policy for each player is a map which for any state assigns a probability distribution over all possible actions at that state.
Namely a policy for \textit{Max} is a map 
$\pi:S \to \cup_{s\in S} \Delta(A^1(s))$ such that $\pi(s)\in \Delta(A^1(s))$
for all $s\in S$, and a policy for \textit{Min} is a map
$\sigma:S \to \cup_{s\in S} \Delta(A^2(s))$ satisfying the analogous condition.

A game trajectory (state action path) is now of the form:
\begin{equation}
\label{game trajectory}
\alpha\coloneqq(s_t,a^1_t,a^2_t,s_{t+1},
\dots, s_{T-1},a^1_{T-1},a^2_{T-1},s_T),
\quad s_t,\dots,s_T \in S, \; a^1_k,a^2_k \in A(s_k) \text{ for }t\leq k\leq T-1\enspace .
\end{equation}
The associated discounted reward is given by:
    \begin{align}\label{game val pol}
      r(\alpha)&\coloneqq
                 \sum_{k=t}^{T-1}(\prod_{l=t}^{k-1}\gamma_l)r_k(s_k,a_k^1,a_k^2) \;+\; (\prod_{l=t}^{T-1}\gamma_l)\phi (s_T)\enspace .
 \end{align}
Given a sequence of randomized policies for the \textit{Max} player, $\bm{\pi}=(\pi_1,\dots,\pi_{T-1})$ and  a sequence for the \textit{Min} player,
$\bm{\sigma}=(\sigma_1,\dots,\sigma_{T-1})$ we define a probability measure on
the set of trajectories, analogously to \eqref{prob MDP path},
as 
\begin{equation}
\label{game path prob}
P^{\bm{\pi},\bm{\sigma}}(\alpha|s_t)\coloneqq\prod_{i=t}^{T-1}\pi_i(a^1_i|s_i)\sigma_i(a^2_i|s_i)P_t(s_{i+1}|s_i,a^1_i,a^2_i) \enspace .
\end{equation}
We then
define $V^{t,\bm{\pi},\bm{\sigma}}_s$, the stage $t$ \textit{value function} evaluated at state $s$, under the sequence of policies
$\bm{\pi}$ and $\bm{\sigma}$, to be the expected value with respect to the probability measure \eqref{game path prob}, of the sum of discounted rewards along all paths from state $s$ till the end of the game, analogously to \eqref{MDP value given policy},
\begin{equation}
  \label{game value given policy}
V^{t,\bm{\pi},\bm{\sigma}}_s
\;=\;
\mathbb{E}^{\bm{\pi},\bm{\sigma}}_s r(\alpha)
=\sum_\alpha P_t^{\bm{\pi},\bm{\sigma}}(\alpha|s)r(\alpha) 
\enspace .
\end{equation}
We then  have the Kolmogorov recursion:
 \begin{equation}
V^{t,\bm{\pi},\bm{\sigma}}_s
= \sum_{a^1\in\actions^1(s),a^2 \in \actions^2(s)} \pi_t(a^1|s)\sigma_t(a^2|s)
  \Bigl[
    r(s,a)
    + \sum_{s'\in\states} \gamma_t(s) P_t(s'\mid s,a^1,a^2)\,V^{t+1,\bm{\pi},\bm{\sigma}}_{s'}
  \Bigr],
\end{equation}

Notice that up to this point what we have defined is entirely analogous to the MDP case.

The game starting from state $s$
at stage $t$ has a \emph{value} $V_s^t$, and there is a pair $(\bm{\pi}^*,\bm{\sigma}^*)$ of optimal sequences of policies, meaning that
\begin{align}
  \label{e-def-val}
  V_s^{t,\bm{\pi},\bm{\sigma}^*} \leq V^t_s\coloneqq  V_s^{t,\bm{\pi}^*,\bm{\sigma}^*} \leq V_s^{t,\bm{\pi}^*,\bm{\sigma}}\enspace, \qquad s\in S,\; 1\leq t\leq T\enspace,
\end{align}
for all pairs of sequences of policies $(\bm{\pi},\bm{\sigma})$.
In other words, $(\bm{\pi}^*,\bm{\sigma}^*)$ is a saddle point.
In particular,
\begin{equation}
\label{e-def-val2}
V^t_s
=\max_{\bm{\pi}}\;\min_{\bm{\sigma}}\;V^{t,\bm{\pi},\bm{\sigma}}_s
=\min_{\bm{\sigma}}\;\max_{\bm{\pi}}\;V^{t,\bm{\pi},\bm{\sigma}}_s \enspace .
\end{equation}
or more explicitly
  \begin{equation}
\label{Val game}V^t_s
\;=\;
\max_{\bm{\pi}}\;\min_{\bm{\sigma}}
\;\mathbb{E}^{\bm{\pi},\bm{\sigma}}\Biggl[\sum_{k=t}^{T-1} (\prod_{l=t}^{k-1}\gamma_l(s_l))r_k\bigl(s_k,a^1_k,a^2_k\bigr)
\;+\;(\prod_{l=t}^{T-1}\gamma_l(s_l))\phi(s_T)
\;\Bigm|\;s_t=s\Biggr], \qquad s\in S,\;1\leq t\leq T-1
\end{equation}
in which the max and the min commute.
Moreover, the value satisfies the Shapley-Bellman equation
\begin{equation}\label{e-Shapley}
V^t_s
=\max_{\bm{\pi}}\;\min_{\bm{\sigma}}
\sum_{a^1,a^2}\pi_t(a^1|s)\,\sigma_t(a^2|s)\Bigl[r_t(s,a^1,a^2)
\;+\;
\sum_{s'\in S}\gamma_t(s)P_t\bigl(s'\mid s,a^1,a^2\bigr)\,V^{t+1}_{s'}\Bigr],
\end{equation}
where the max and the min commute, together with
the boundary condition
\[
V^T_s=\phi(s), 
\qquad s \in S \enspace .
\]
Furthermore, an optimal policy of Player \emph{Max} at stage $t$ is obtained by selecting a policy $\pi$ which achieves the maximum in~\eqref{e-Shapley}, whereas an optimal policy for player \emph{Min} is obtained dually.
Results of this nature go back to Shapley~\cite{shapley_stochastic}, building on von Neumann's minimax theorem, we refer
the reader to~\cite[Th.~IV.3.2 p.~182]{sorin_repeated_games} or~\cite{solan}
for a proof.

If the second player is a dummy, meaning that in every state $s$, $A^2(s)$
is a singleton, the concurrent game reduces to a MDP.

Just like the Bellman equation for the value of an MDP is a discretized  Hamilton-Jacobi-Bellman equation so the Shapley-Bellman equation for the value of a game, is a discretized version of the Hamilton-Jacobi-Isaacs equation.

\if{\subsection{One player games, aka MDP}
Notice now that indeed, if the second player has always just a single action then this reduces to an MDP. 
We show this next.

Starting from the finite‐horizon Shapley recursion for two players,
\[
V^t_s
=\max_{a^1\in A^1(s)}
\;\min_{a^2\in A^2(s)}
\Bigl\{
r_t(s,a^1,a^2)
+\sum_{s'}\gamma_tP\bigl(s'\mid s,a^1,a^2\bigr)\,V^{t+1}_{s'}
\Bigr\},
\quad V^T_s=\phi(s),
\]
assume there is only one decision‐maker.  Then
\[
A^2(s)=\{\star\},
\]
so the inner minimization over \(a^2\) is trivial:
\[
\min_{a^2\in\{\star\}}
\bigl\{r_t(s,a^1,a^2)+\sum_{s'}\gamma_T P(s'\mid s,a^1,a^2)V^{t+1}_{s'}\bigr\}
=r_t(s,a^1,\star)+\sum_{s'}\gamma_t p(s'\mid s,a^1,\star)V^{t+1}_{s'}.
\]
Define
\[
r_t(s,a)\equiv r_t(s,a,\star),
\qquad
P(s'\mid s,a)\equiv P(s'\mid s,a,\star).
\]
Hence the recursion reduces to
\[
V^t_s
=\max_{a\in A(s)}
\Bigl\{
r_t(s,a)
+\sum_{s'}\gamma_t P(s'\mid s,a)\,V^{t+1}_{s'}
\Bigr\},
\quad V^T_s=\phi(s),
\]
which is the finite‐horizon Bellman equation for a Markov Decision Process \eqref{MDP Bellman}.

}\fi
\subsection{Turn-based games}
\emph{Turn-based games}
are special concurrent games in which at any given state,
only one player has a non-trivial choice of action.
Therefore, the state space can be partitioned as
\[
S = S^1 \;\cup\; S^2,
\]
where
\begin{itemize}
  \item[--] $S^1$ are the states where Player 1 plays,
  \item[--] $S^2$ are the states where Player 2 plays.
\end{itemize}

\paragraph{\textbf{Player \textit{Max's} turn.}}

Since Player 2 has no choice at $s\in S^1$, the inner minimization
in the Shapley-Bellman equation~\eqref{e-Shapley} is trivial, therefore
\begin{equation}
\label{Turnvalmax}
V^t_s
=
\max_{a^1\in A^1(s)}
\Bigl\{
r_t(s,a^1)
+\sum_{s'\in S^1 \cup S^2}\gamma_t(s) P_t\bigl(s'\mid s,a^1\bigr)\,V^{t+1}_{s'}
\Bigr\},
\quad s\in S^1.
\end{equation}

\paragraph{\textbf{Player \textit{Min's}  turn.}}
Similarly, Player 1 has no choice at $s\in S^2$ therefore
\begin{equation}
\label{Turnvalmin}
V^t_s
=
\min_{a^2\in A^2(s)}
\Bigl\{
r_t(s,a^2)
+\sum_{s'\in S^1 \cup S^2}\gamma_t(s) P_t\bigl(s'\mid s,a^2\bigr)\,V^{t+1}_{s'}
\Bigr\},
\quad s\in S^2.
\end{equation}

    Note that these two equations are coupled as an action can move from a state where Player 1 plays to a state where Player 2 plays. 
    In other words even though $s$ may be in $S^1$, $P(s'|s,a_1)$ can be non zero for $s'\in S^2$. Therefore  $V^t_s$ for $s\in S^1$  may depend on $V^{t+1}_{s'}$ for $s'$ in $S^2$. 

We can also compute the optimal policy realizing the value of the game. This is analogous to the MDP case we saw in \eqref{MDP opt policy}. Unlike in the general concurrent game, the optimal policies are deterministic.
Indeed, an optimal policy for the Max player is characterized by
\begin{equation}
\label{Game opt policy 1}
    \pi^*_t(s)\in \arg\max_a 
      \Bigl[r(s,a)
    + \sum_{s'\in\states} \gamma_t(s) P_t(s'\mid s,a)\,V^{t+1}_{s'}\Bigr]
\end{equation}
where $s\in S^1$. Similarly, an optimal policy for the Min player 
satisfies
\begin{equation}
\label{Game opt policy 2}
    \sigma^*_t(s)\in \arg\min_a 
      \Bigl[r(s,a)
    + \sum_{s'\in\states} \gamma_t(s) P_t(s'\mid s,a)\,V^{t+1}_{s'}\Bigr]
\end{equation}
where $s\in S^2$.
\subsection{Stopping games}
Finally we will need \textit{stopping turn based games}. This means that at any state there are only two possible actions: stop or continue. The stop action is realized by introducing a so-called cemetery state, denoted $\bot$, which is absorbing.  This means that once there, a player stays there forever, receiving no reward.  To choose the stopping action means to go to the cemetery state $\bot$. 

\section{ReLU neural net as a turn-based, stopping game}\label{sec-turn}
    We will now show that \textit{the output of a ReLU neural net is the same as the value of a two-player,  zero-sum, turn-based, stopping game, which we call the ReLU net game}. The game runs in the opposite direction with respect to the neural net. 

    \subsection{Reminder on ReLU neural nets}
Recall $\operatorname{ReLU}(x)\coloneqq\max(x,0)$.
Assume the network has $L$ layers and layer $l$ has $k_l$ neurons. 
We number the layers starting from the output of the neural network (layer $1$) to its input (layer $L$).
The weight matrix of layer $l$ is denoted by $W^l$ and there are bias vectors $b^l \in \mathbb{R}^{k_l}$ in each layer $l$. 
The input vector is $x\in \mathbb{R}^{k_L}$.

Therefore the total output function of the net is $f:\mathbb{R}^{k_L} \to \mathbb{R}^{k_1}$ where
\begin{equation}
\label{NNo-bis}
 f(x)=\max(W^1\max(W^{2}(\max \dots (\max(W^{L-1}(\max(W^Lx+b^L,0))+b^{L-1},0)\dots)) +b^{2},0)+b^1,0)
\end{equation}
and the $\max$  is applied coordinate-wise on a vector. 

The reason for numbering the layers from the output layer of the net, to the input is so that time will move forward along the game and the Shapley-Bellman backward recursion will move from the input of the net to its output. 

The map $f(x)$ is a (generally non-convex) piecewise-linear map.

The formula which we will get from the game theoretic perspective will naturally be a maxmin formula.
\subsection{Description of the ReLU net game}

Given a ReLU net as described above, we define now a two-player, zero-sum turn-based, stopping game that we call the \textit{ReLU net game}. The game is played in the \textbf{opposite} of the direction  the neural network is running. %

The players are called \textit{Max} and \textit{Min} since one is trying to maximize reward and the other is trying to minimize it. 
At the end of a game played, the total reward is paid to the \textit{Max} player by the \textit{Min} player. 

The input to the neural net is the terminal reward for the game.
The expected value of rewards under optimal play, (for given terminal reward) will give the output of the ReLU neural network,

The biases of the neural net will be used to define rewards. The weights will be normalized in order to define probabilities thereby introducing discount factors. 
Concretely:
 \begin{itemize}
\item[--] 
The game starts at the last layer  of the neural net (layer 1) and proceeds up towards the first layer of the neural net (layer $L$) where it ends.
So \textbf{every layer corresponds to one time stage of the game}. 
Therefore the Shapley-Bellman backward recursion starts at the beginning of the neural net (layer $L$) and proceeds towards the end of the neural net (layer 1). 
\item[--]
The \textbf{states} of the game are as follows: Every neuron in a layer  of the net gives rise  to two game states called positive and negative.
We denote neuron $i$ in layer $l$ (counting from the end of the neural net) by $(l,i)$. 
To this neuron correspond two game states, $(l,i+)$ at which $\textit{Max}$ plays  and 
$(l,i-)$ at which $\textit{Min}$ plays. 
Denote the states where \textit{Max} plays by $S^+$ and the states where \textit{Min} plays by $S^-$.

Moreover there is a so-called \textbf{cemetery state} denoted $\bot$, which is absorbing (meaning once there you stay there forever) and corresponds to stopping playing. The future value from that state is always zero since the instantaneous payoff in this state is zero.

\item[--] 
  The \textbf{reward} at state $(l,i +)$ is $r^l(l,i+)\coloneqq b^l_i$, and the  reward  in state $(l,i -)$ is
  $r^l(l,i-)\coloneqq-b^l_i$,  where $b^l_i$ is the bias at neuron $(l,i)$.
\item[--]
  The \textbf{terminal reward} in state $(L,i+)$ is $\phi_{L,i+}(x)\coloneqq x_i$, where $x_i$ is the value of the $i$th input to the neural network, and the terminal reward in state $(L,i-)$ is $\phi_{L,i-}(x)\coloneqq -x_i$. 
\item[--]
  Let $\gamma^l_i\coloneqq\sum_j|W^l_{i,j}|$ be the \textbf{discount factor}
  in states $s=(l,i\pm$), so that $\gamma^l(s)\coloneqq \gamma_i^l$. We allow the discount factors to take values greater than $1$. We assume that $\gamma^l_i$ never vanishes.
\item[--] 
The \textbf{action} choice of the player when it is their turn to play, is either to stop which means going to the cemetery state $\bot$, or to continue.
\item[--]
A \textbf{state transition} is to move from a state in one layer to a state in the next layer. 
We define the \textbf{transition probabilities}
as follows.
Recall that
$a^+\coloneqq\max(a,0)\geq 0$ and $a^-\coloneqq\max(-a,0)\geq 0$ so that $a=a^+-a^-$ and $|a|=a^++a^-$. 

Define 
\begin{equation}
\label{game prob}
P^l_{i+,j+}=P^l_{i-,j-}\coloneqq\frac{(W^l_{i,j})^+}{\gamma^l_i}
\text{ and }
P^l_{i+,j-}=P^l_{i-,j+}\coloneqq\frac{(W^l_{i,j})^-}{\gamma^l_i}
\end{equation}

Note that 
\begin{equation}
    \sum_{W^l_{i,j}\geq 0 } P^l_{i+,j+}+\sum_{W^l_{i,j}\leq 0}  P^l_{i+,j-}=
    \sum_{W^l_{i,j}\geq 0 }\frac{1}{\gamma^l_i}(W^l_{i,j})^+ +
   \sum_{W^l_{i,j}\leq 0 } \frac{1}{\gamma^l_i}(W^l_{i,j})^-=
    \frac{1}{\gamma^l_i}\sum_j |W^l_{i,j}|=1 \enspace,
\end{equation}
so that $P^l_{i+,-}$ defines a probability distribution and analogously for $P^l_{i-,-}$.

If a player chooses to continue the game and is in state $(l,i_+)$, then, if $W^l_{i,j}> 0$, they transition to state $(l+1,j_+)$ with probability $P^l_{i+,j+}$; if $W^l_{i,j}< 0$ they transition to state $(l+1,j_-)$ with probability $P^l_{i+,j-}$.\todo{SG: Put action befores transitions, the latter are meaningfull only if cont is chosen}
Similarly, if the current state is $(l,i_-)$, then if $W^l_{i,j}> 0$, they transition to state $(l+1,j_-)$ with probability $P^l_{i-,j-}$; if $W^l_{i,j}< 0$ they transition to state $(l+1,j_+)$ with probability $P^l_{i-,j+}$.

\end{itemize}
If Players $\textit{Max}$ and $\textit{Min}$ play according to the sequences of policies $\bm{\pi}\coloneqq(\pi_1,\dots. \pi_L)$ and $\bm{\sigma}\coloneqq(\sigma_1,\dots , \sigma_L)$, respectively,
the expected payoff received by Player Max in the game from time $l$ to time $L$, with
initial state $s$ of the form $(l,i\pm)$, is given, according to~\eqref{game val pol}, \eqref{game value given policy}, by:
\begin{equation}
  V_{s}^{l,\bm{\pi},\bm{\sigma}}(x) =
  E^{\bm{\pi},\bm{\sigma}}\Big(r^l(s_l)+\gamma_l(s_l) r^{l+1}_{s_{l+1}}+ \dots + (\prod_{k=l}^{L-2}\gamma_k(s_k)) r^{L-1}_{s_{l-1}} + (\prod_{k=l}^{L-1} \gamma_k(s_k)) \phi_{L,i_L}(x) |s_l=s\Big)
\label{e-def-vpisigma}
\end{equation}
\todo{SG: minor discrepancy, I used $E_s$ but here the conditional expectation is noted $E(\cdot|s_t=t)$ can be fixed later on, the conditional expectation notation is quite clear...}

We denote by $V^l_{s}$ 
\begin{align} V^l_{s}(x) =\max_{\bm{\pi}}\;\min_{\bm{\sigma}}
V_{s}^{l,\bm{\pi},\bm{\sigma}}(x)=\min_{\bm{\sigma}}\;\max_{\bm{\pi}}
V_{s}^{l,\bm{\pi},\bm{\sigma}}(x)\label{e-def-value}
\end{align}
the value of the associated zero-sum game, as per~\eqref{e-def-val},\eqref{e-def-val2}.

\begin{remark} 
\textbf{Notation:} To avoid cluttering the notation, if $s=(l,i+)$, instead of writing $V^{l,\bm{\pi},\bm{\sigma}}_{(l,i+)}(x)$ or $V^l_{(l,i+)}(x)$ we simply write $V^{l,\bm{\pi},\bm{\sigma}}_{i+}(x)$ or $V^l_{i+}(x)$,  respectively. The same goes when $s=(l,i-)$.
\end{remark}

For a description of the ReLU net -- ReLU game correspondence, see also \Cref{tab:NN-Games1}.

\begin{proposition}

    Given a ReLU neural net, consider the ReLU net game defined above,  then the values of the game,  $V_{i+}^l$ and  $V_{i-}^l$, satisfy the following Shapley-Bellman equations:
    \begin{equation}
        V^l_{i+}(x)=\max(0,\gamma^l_i [\sum_{W_{i,j} \geq 0} P^l_{i+,j+} V^{l+1}_{j+}(x) +
        \sum_{W_{i,j} \leq 0} P^l_{i+,j-} V^{l+1}_{j-}(x)]
        +b^l_i)\label{e-Val plus}
    \end{equation}

    \begin{equation}
        V^l_{i-}(x)=\min(0,\gamma^l_i [\sum_{W_{i,j} \geq 0} P^l_{i-,j-} V^{l+1}_{j-}(x) +
        \sum_{W_{i,j} \leq 0} P^l_{i-,j+} V^{l+1}_{j+}(x)]
        -b^l_i)\label{e-Val minus}
    \end{equation}
with boundary conditions 
$V^L_{i,+}(x)\coloneqq x_i$ and $V^L_{i,-}(x)\coloneqq-x_i$
    where the vector $x\coloneqq(x_1,\dots,x_{n_1})$ is the input to the network.
\end{proposition}

\begin{proof}

We apply equations \eqref{Turnvalmax} and \eqref{Turnvalmin} already derived for a turn-based game.
The stopping decision leads to the cemetery state and zero value. The decision to continue from $(l,i,+)$ has two kinds of possible destinations $(l+1,j,+)$ if $W_{i,j}>0$ and $(l+1,j,-)$ if $W_{i,j}<0$. Therefore the result follows. The proof for $V^l_{i-}$ is analogous. 
\end{proof}
\begin{remark}
    Note that substituting ~\eqref{e-Val minus} in ~\eqref{e-Val plus} we get an expression with both $\max$ and $\min$. 
    
\end{remark}
\begin{theorem}
\label{theorem 1}
  The value of the $i$th-output of a ReLU neural network of depth $L$, on input vector $x$, coincides with the value of the associated discounted turn-based stopping game in horizon $L$ with initial state $(1,i,+)$ and terminal payoff $\phi_{L,\cdot}(x)$.

  More precisely,  let $y^l_i$ be the output of the $i$th neuron in the $l$ layer of the neural network,
  so that the output of the neural net is $y^1$ and the input is $y^L$ where $L$ is the number of layers of the neural net; then
\begin{align}
\label{Val=NN out}
y^l_i = V^l_{i,+}(x)= -
V^l_{i,-}(x) \enspace .
\end{align}
\end{theorem}
\begin{proof}

We will prove this by induction.
Indeed by definition  for the terminal rewards we have  $V^L_{i+}(x)=x_i=y^L_i$ and $V^L_{i-}(x)=-x_i=-y^L_i$. 

Moreover, assume
$V^{l+1}_{j+}(x)=y^{l+1}(x)$ and $V^{l+1}_{j-}(x)=-y^{l+1}(x)$ then
we see from the  Shapley-Bellmann equation \eqref{e-Val plus} that
\begin{align}
    V^l_{i+}(x)& =\max(0,\gamma^l_i [\sum_{W_{i,j} \geq 0} P^l_{i+,j+} V^{l+1}_{j+}(x) +
        \sum_{W_{i,j} \leq 0} P^l_{i+,j-} V^{l+1}_{j-}(x)]
        +b^l_i)
        \\
&=\max(0,\gamma^l_i [\sum_{W_{i,j} \geq 0} P^l_{i+,j+} y^{l+1}_{j-}(x) +
        \sum_{W_{i,j} \leq 0} P^l_{i+,j-} (-y^{l+1}_j(x))]
        +b^l_i)
  \\
  &=
       \max(0,\gamma^l_i \sum_j [(W^l_{i,j})^+- (W^l_{i,j})^-]y^{l+1}_{j-}(x)
        +b^l_i)=
        \max(0, \sum_j W^l_{i,j}y^{l+1}_j +b^l_i)=y^l_i\enspace .
        \end{align}  
The proof that $V^l_{i-}(x)=-y^l_i$ is dual.
\end{proof}
\begin{remark}
Note that if all the weights are positive then we have a one player game, namely a MDP.
\end{remark}
The following observation allows us to relate the Lipschitz constant of the ReLU net map
with the discount factors of the game.
\begin{proposition}[Lipschitz constant of the ReLU net map]
The map computed by a ReLU neural net is Lipschitz with respect to the sup norm. Moreover 
for all $1\leq l\leq L$, consider the maximal discount factor $\bar{\gamma}^l\coloneqq  \max_{1\leq i\leq k_l}\gamma_i^l$. Then, the product $\bar{\gamma}^1\dots \bar{\gamma}^L$
  provides an upper bound of the Lipchitz constant of the neural-network map $f$ with respect to the sup-norm.
\end{proposition}
\begin{proof}
  We observe that each of the affine maps arising at the right-hand-side of~\eqref{e-Val plus}-\eqref{e-Val minus} is Lipschitz of constant at most $\bar{\gamma}^l$, since, by H\"older inequality,
  the Lipschitz constant in the sup-norm of an affine map is the $L_1$-norm of its gradient.
  Moreover, the set of functions
  that are Lipschitz for a common constant is stable by infimum and supremum.
  Then, it follows from~\eqref{e-Val plus}-\eqref{e-Val minus} that every
  value $V_{i\pm}^l(x)$ is obtained by applying a Lipschitz function of constant $\bar{\gamma}^l$ to the vector of values $(V_{k\epsilon}^{l+1}(x))_{1\leq k\leq k_{l+1}, \epsilon \in \pm }$.
The conclusion follows by composing these Lipschitz functions.
  \end{proof}

    \begin{remark}
\Cref{theorem 1} shows that the game encodes in a self-dual way both the ReLU map $x\mapsto y$ and its opposite. Indeed,  $-y_i^l$ is gotten by considering the value function at the the states $(l,i-)$, instead of $(l,i+)$. 
  \end{remark}
\subsection{Example of a ReLU neural net as a game}
Consider a three-layer
ReLU neural net with 2 neurons in the input layer, 2 neurons in the intermediate layer and 1 neuron in the last (output) layer. 
We assume the input is  $x=(x_1,x_2)$.
Let 
$y^2_1$ and $y^2_2$ be the output of the middle layer and $y^1$ the output of the net.

Let's assume that the weight matrix from the input layer to the intermediate layer is
\[
W^2\coloneqq\begin{pmatrix}
7 & -8 \\
-1 & -2
\end{pmatrix}
\]
Moreover the biases are
\[
b^2=\begin{pmatrix}
42\\
33
\end{pmatrix}
\]
The weight matrix from the intermediate layer to the output layer, and the bias, are given by
\[
W^1\coloneqq(2,-5), \qquad
b^1\coloneqq7.
\]

We see that the output of the middle layer is given by
\begin{align}\label{e-middle}
\begin{pmatrix}
y^2_1\\
y^2_2
\end{pmatrix}=
\max(
W^2 \begin{pmatrix}
x_1\\
x_2
\end{pmatrix}+
\begin{pmatrix}
    b^2_1\\
    b^2_2
\end{pmatrix},0)
=\begin{pmatrix}
\max(7x_1-8x_2+42,0)&
    \\
\max(-x_1-2x_2+33,0),
\end{pmatrix}
\end{align}
and the output of the net is given by
\begin{equation}
y^1=\max(W^1
\begin{pmatrix}
    y^2_1\\
    y^2_2
\end{pmatrix}+
    b^1,0)
\end{equation}
 and therefore
 \begin{equation}
 \label{e-first}
   y^1=
    \max(7+2y^2_1-5y^2_2,0)
    \end{equation}

Let us now construct the corresponding ReLU net game.
We have the discount factors: 
$\gamma^2_1=7+8=15$,
 $\gamma^2_2=1+2=3$
and $\gamma^1=2+5=7$ and 
thus
\[
(W^2)^+\coloneqq\begin{pmatrix}
7 & 0 \\
0 & 0
\end{pmatrix}
\]
\[
(W^2)^-\coloneqq\begin{pmatrix}
0 & 8 \\
1 & 2
\end{pmatrix}.
\]
Therefore the transition probabilities where the same player keeps playing are 
\[
\begin{pmatrix}
P^2_{1+,1+}=P_{1-,1-} & P^2_{1+,2+}=P^2_{1-,2-} \\
P^2_{2+,1+}=P^2_{2-,1-} & P^2_{2+,2+}=P^2_{2-,2-}
\end{pmatrix}=
\begin{pmatrix}
\frac{7}{15} & 0 \\
0 & 0
\end{pmatrix}
\]
and the transition probabilities where the player changes are

\[
\begin{pmatrix}
P^2_{1+,1-}=P^2_{1-,1+} & P^2_{1+,2-}=P^2_{1-,2+} \\
P^2_{2+,1-}=P^2_{2-,1+}  & P^2_{2+,2-}=P^2_{2-,2+} 
\end{pmatrix}=
\begin{pmatrix}
0 & \frac{8}{15} \\
\frac{1}{3} & \frac{2}{3}
\end{pmatrix}
\]

Moreover $(W^1)^+=(2,0)$ and
$(W^1)^-=(0,5)$.

This means that 
\begin{align}
P^1_{1+,1+}=P^1_{1-,1-}=\frac{2}{7} &  , \qquad
P^1_{1+,2-}=P^1_{1-,2+}=\frac{5}{7}. 
\end{align}
while
\begin{align}
P^1_{1+,1-}=P^1_{1-,1+}=0 
\end{align}

Now we write the Shapley equations~\eqref{e-Val plus},\eqref{e-Val minus}. Recall that the boundary conditions are $V^3_{1+}(x)=x_1$, $V^3_{2,+}(x)=x_2$, $V^3_{1-}(x)=-x_1$ and $V^3_{2,-}(x)=-x_2$.
We then have 
\begin{align*}
V^2_{1+}(x)&= \max(0,\gamma^2_1(P^2_{1+,1+}V^3_{1+}(x)+P^2_{1+,2-}V^3_{2-}(x))+b^2_1) \iff
y^2_1 =\max(0,15(\frac{7}{15}x_1+\frac{8}{15}(-x_2))+42)\\
V^2_{2+}(x)(x)&=\max(0,\gamma^2_2(P^2_{2+,1-}V^3_{1-}(x)+P^3_{2+,2-}V^3_{2-}(x))+b^1_2) \iff
y^2_2=\max(0,3(\frac{1}{3}(-x_1)+\frac{2}{3}(-x_2))+33)
\end{align*}
which is the same as equation~\eqref{e-middle} of the neural net.
Similarly
\begin{equation}
V^1_{1+}(x)=\max(0,\gamma^1(P^1_{1+,1+}V^2_{1+}(x)+P^1_{1+,2-}V^2_{2-}(x))+b^1) \iff
y^1=\max(0,7(\frac{2}{7}y^2_1+\frac{5}{7}(-y^2_2))+7)
\end{equation}
is the same as  \eqref{e-first}.

For the states belonging to the $\min$ player, we check for example that 
\begin{equation}
V^2_{1-}(x)=\min(0,\gamma^2_1(P^2_{1-,1-}V^3_{1-}(x)+P^2_{1-,2+}V^3_{2+}(x))-b^2_1) \iff
-y^2_1=\min(0,15(\frac{7}{15}(-x_1)+\frac{8}{15}x_2)-42)
\end{equation}
Moreover
\begin{equation}
V^2_{2-}(x)=\min(0,\gamma^2_2(P^2_{2-,1+}V^3_{1+}(x)+P^2_{2-,2+}V^3_{2+}(x))-b^2_2) \iff
-y^2_2=\min(0,3(\frac{1}{3}x_1+\frac{2}{3}x_2)-33)
\end{equation}

Figure~\ref{fig:graph-of-the-game} shows the game corresponding to this neural net. 

\begin{figure}[ht]
\centering
\begin{tikzpicture}[
  scale=1.6,
  transform shape,
  >=Stealth,
  node distance=22mm and 28mm,
  vertex/.style={draw,circle,minimum size=1mm,inner sep=0pt,font=\scriptsize},
  edgelabel/.style={font=\scriptsize,sloped,allow upside down},
  labelbg/.style={fill=white,inner sep=1pt},
  edgew/.style={font=\scriptsize,sloped,allow upside down,fill=white,inner sep=1pt,text opacity=1},
  diamondnode/.style={draw,diamond,aspect=2,inner sep=0.6pt,minimum size=3.5mm},
  squarenode/.style={draw,rectangle,inner sep=0.6pt,minimum size=3.2mm},
]

\node[vertex] (v11p) at (0,  1.2) {\tiny$\,1,1+$};
\node[vertex] (v11m) at (0, -1.2) {\tiny$\,1,1-$};

\node[vertex] (v21p) at (4,  1.5) {\tiny$\,2,1+$};
\node[vertex] (v22p) at (4,  0.5) {\tiny$\,2,2+$};
\node[vertex] (v21m) at (4, -0.5) {\tiny$\,2,1-$};
\node[vertex] (v22m) at (4, -1.5) {\tiny$\,2,2-$};

\node[vertex] (v31p) at (8,  1.5) {\tiny$\,3,1+$};
\node[vertex] (v32p) at (8,  0.5) {\tiny$\,3,2+$};
\node[vertex] (v31m) at (8, -0.5) {\tiny$\,3,1-$};
\node[vertex] (v32m) at (8, -1.5) {\tiny$\,3,2-$};

\node[diamondnode] (d11p) at (0.8,  1.2) {};
\node[squarenode]  (s11p) at (1.6,  1.2) {};
\draw[->] (v11p) -- (d11p);
\draw[->] (d11p) -- (s11p) node[midway,edgelabel,above] {\tiny $7$};
\draw[->] (d11p.south) -- ++(0,-0.6) node[below,font=\scriptsize] {\tiny $0$};

\node[diamondnode] (d11m) at (0.8, -1.2) {};
\node[squarenode]  (s11m) at (1.6, -1.2) {};
\draw[->] (v11m) -- (d11m);
\draw[->] (d11m) -- (s11m) node[midway,edgelabel,below] {\tiny $-7$};
\draw[->] (d11m.south) -- ++(0,-0.6) node[below,font=\scriptsize] {\tiny$0$};

\node[diamondnode] (d21p) at (4.8,  1.5) {};
\node[squarenode]  (s21p) at (5.6,  1.5) {};
\draw[->] (v21p) -- (d21p);
\draw[->] (d21p) -- (s21p) node[midway,edgelabel,above] {\tiny$42$};
\draw[->] (d21p.south) -- ++(0,-0.3) node[below,font=\scriptsize] {\tiny$0$};

\node[diamondnode] (d22p) at (4.8,  0.5) {};
\node[squarenode]  (s22p) at (5.6,  0.5) {};
\draw[->] (v22p) -- (d22p);
\draw[->] (d22p) -- (s22p) node[midway,edgelabel,above] {\tiny$33$};
\draw[->] (d22p.south) -- ++(0,-0.3) node[below,font=\scriptsize] {\tiny$0$};

\node[diamondnode] (d21m) at (4.8, -0.5) {};
\node[squarenode]  (s21m) at (5.6, -0.5) {};
\draw[->] (v21m) -- (d21m);
\draw[->] (d21m) -- (s21m) node[midway,edgelabel,below] {\tiny$-42$};
\draw[->] (d21m.south) -- ++(0,-0.3) node[below,font=\scriptsize] {\tiny$0$};

\node[diamondnode] (d22m) at (4.8, -1.5) {};
\node[squarenode]  (s22m) at (5.6, -1.5) {};
\draw[->] (v22m) -- (d22m);
\draw[->] (d22m) -- (s22m) node[midway,edgelabel,below] {\tiny$-33$};
\draw[->] (d22m.south) -- ++(0,-0.3) node[below,font=\scriptsize] {\tiny$0$};

\draw[->] (s11p) -- (v21p) node[pos=0.14,edgelabel,above,yshift=2pt,xshift=2pt,labelbg] {$\tfrac{2}{7}$};
\draw[->] (s11p) -- (v22m) node[pos=0.14,edgelabel,below,yshift=-2pt,xshift=2pt,labelbg] {$\tfrac{5}{7}$};

\draw[->] (s11m) -- (v22p) node[pos=0.14,edgelabel,above,yshift=2pt,xshift=2pt,labelbg] {$\tfrac{5}{7}$};
\draw[->] (s11m) -- (v21m) node[pos=0.14,edgelabel,below,yshift=-2pt,xshift=2pt,labelbg] {$\tfrac{2}{7}$};

\draw[->] (s21p) to[bend left=8]
    node[pos=0.18,edgew] {$\tfrac{7}{15}$} (v31p);
\draw[->] (s21p) to[bend right=8]
    node[pos=0.10,edgew] {$\tfrac{8}{15}$} (v32m);   %

\draw[->] (s22p) to[bend left=12]
    node[pos=0.12,edgew] {$\tfrac{1}{3}$} (v31m);
\draw[->] (s22p) to[bend right=12]
    node[pos=0.12,edgew] {$\tfrac{2}{3}$} (v32m);

\draw[->] (s21m) to[bend left=12]
    node[pos=0.12,edgew] {$\tfrac{8}{15}$} (v32p);
\draw[->] (s21m) to[bend right=12]
    node[pos=0.12,edgew] {$\tfrac{7}{15}$} (v31m);    %

\draw[->] (s22m) to[bend left=8]
    node[pos=0.10,edgew] {$\tfrac{1}{3}$} (v31p);
\draw[->] (s22m) to[bend right=8]
    node[pos=0.18,edgew] {$\tfrac{2}{3}$} (v32p);

\draw[->] (v31p.east) -- ++(0.6,0) node[right,font=\small] {\tiny$x_{1}$};
\draw[->] (v32p.east) -- ++(0.6,0) node[right,font=\small] {\tiny$x_{2}$};
\draw[->] (v31m.east) -- ++(0.6,0) node[right,font=\small] {\tiny$-x_{1}$};
\draw[->] (v32m.east) -- ++(0.6,0) node[right,font=\small] {\tiny$-x_{2}$};

\end{tikzpicture}
\caption{Graph of the game corresponding to the ReLU neural net in the example in Section 4.3. The circles denote the states. A diamond after a state denotes the 2 possible actions at the state: stop (and get 0 reward) or continue and get the reward denoted on the edge exiting the diamond. A square is the transition to the next state. The edges exiting a square denote the non-trivial choices and the transition probabilities are indicated along these edges. The arrows point at the direction the game is played which is the opposite of the one the neural net is running. Therefore the inputs $(x_1,x_2)$ to the net are the terminal rewards of the  game. The evaluation of the neural net coincides with the Shapley-Bellman backward recursion for the value of the game. }
\label{fig:graph-of-the-game}
\end{figure}
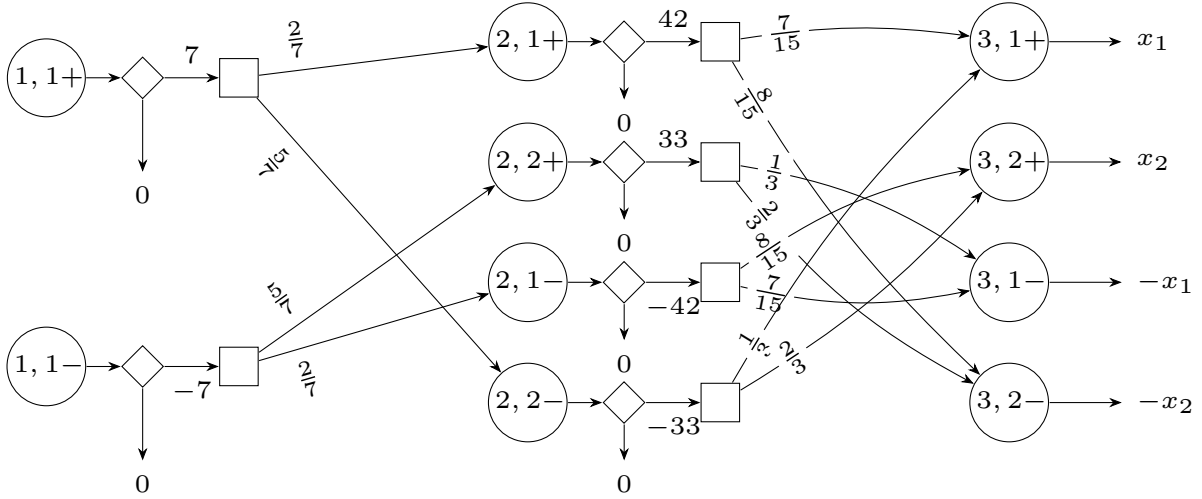

\section{Some applications of the ReLU net/game correspondence}\label{sec-appli}
The fact that a ReLU net can be interpreted as a game has certain implications for our understanding of what it computes and what properties the ReLU net map has. We explain these next.   

\subsection{Discrete path integral interpretation of the map computed by the ReLU net.}

Knowing that the output of the ReLU net for given input is the value of the ReLU net game for that given terminal reward allows us to interpret the output as a certain discrete path integral. The idea is that given the optimal policy, the value of the game is simply the expectation value of accumulated rewards along game trajectories.
To explain this we need some notation:

\begin{definition}
\label{Def tech}
Let
$\alpha:\{l,l+1,\dots,l+k\}\to S^+\cup S^-\cup\{ \bot\}$ be a game trajectory starting from a state at stage $l$ of the game.
Let $\nu\coloneqq 0\dots k-1$.
\begin{itemize}
    \item[--]
   We  define  $\sgn(\alpha(l+\nu))\coloneqq+1$ if $\alpha(l+\nu)$  is a Max state 
and $\sgn(\alpha(l+\nu))\coloneqq-1$ if $\alpha(l+\nu)$ is a Min state.
\item[--]
Denote by $0$ the stopping action and by $1$ the continue action.
\item[--]
A policy for Max is a 
map 
$\bm{\pi}:S^+ \to \{0,1\}$ and 
a policy for Min is a map 
$\bm{\sigma}:S^- \to \{0,1\}$.
We also have $\bm{\pi}(\bot)=\bm{\sigma}(\bot)\coloneqq 0$.

So a pair of policies is simply an assignment of $0$ or $1$ to the game states (where $\bot$ is always assigned $0$).

\end{itemize}
\end{definition}
\begin{definition}
    \label{Path space}
Let $\Path^{\bm{\pi},\bm{\sigma}}_{(l,i+)}$ denote the set of game trajectories (paths) 
which start at $(l,i+)$ and a length compatible with the policies $\bm{\pi}$ and $\bm{\sigma}$. To be precise let 
$\Path^{\bm{\pi},\bm{\sigma}}_{(l,i+)}$ be the set of 
$\alpha:\{l,l+1,\dots,l+k\}\to S^+\cup S^-\cup \{\bot\}$ satisfing the following three conditions:
\renewcommand{\theenumi}{\roman{enumi}}
\begin{enumerate}
    \item
$\alpha(l)=(l,i+)$ 
\item
For $\nu\coloneqq 1\dots k-1$,  %
if $\alpha(l+\nu)$ is a Max state we have 
$\bm{\pi}(\alpha(l+\nu))=1$ and
if  $\alpha(l+\nu)$ is a Min state we have 
$\bm{\sigma}(\alpha(l+\nu))=1$. 
\item
if $\alpha(l+k)$ is a Max state we have 
$\bm{\pi}(\alpha(l+k))=0$ and
if  $\alpha(l+k)$ is a Min state we have 
$\bm{\sigma}(\alpha(l+\nu))=0$.
Moreover we call that  $k$ the length of $\alpha$ and we define 
$\len(\alpha)\coloneqq k$. 
Since the horizon is $L$ we have $l+k \leq L$.
\end{enumerate}
The conditions (i)--(iii) imply that $l+k$ is the 
stage at which the stopping action is exercised.
To simplify the notation in what follows, if $\alpha(l)=(l,i\pm)$, we set $\gamma^l\coloneqq \gamma^l_i$.\todo{SG: to be fixed later on, homogeneized}
 
 \end{definition}
  \begin{proposition}
    Let $x\coloneqq(x_1,\dots x_{n_L})$ be the input to the ReLU net.
\todo{SG: deleted    ``Consider $x$ as the final reward for the corresponding ReLU net game.'', already said and we have the $-x$ issue}
    Then
      the value $V^{l,\bm{\pi},\bm{\sigma}}_{i,+}(x)$ under the policy sequences $\bm{\pi}$ and $\bm{\sigma}$ is given by
      \begin{equation}
         \label{V given pol}
V^{l,\bm{\pi},\bm{\sigma}}_{i,+}(x)=\sum_{\alpha \in \Path^{\bm{\pi},\bm{\sigma}}_{(l,i+)}} r(\alpha)
\prod_{\nu=0}^{\len(\alpha)-1} 
P^{l+\nu}_{\alpha(l+\nu),\alpha(l+\nu+1)}  
\end{equation}
where for $\alpha \in \Path^{\bm{\pi},\bm{\sigma}}_{(l,i+)}$
\begin{equation}
r(\alpha)\coloneqq
\sum_{\nu=0}^{\len(\alpha)-1}
\sgn((\alpha(l+\nu))\left(\prod_{\mu=l}^{l+\nu-1}  \gamma^\mu\right)
b^{l+\nu}_{\alpha(l+\nu)} 
+\delta_{\len(\alpha),L}
\left(\prod_{\mu=l}^{L-1}  \gamma^\mu\right)
\Phi_{\alpha(L)}(x)
\end{equation}
and $\delta_{\len(\alpha),L}=1$ if $\len(\alpha)=L$ and $0$ otherwise.
Consequently the output $y^l_i$ of neuron $(l,i)$ is given by 
\begin{equation}\label{e-path2}
y^l_i=V^l_{i+}(x)=\max_{\bm{\pi}}\min_{\bm{\sigma}}V^{l,\bm{\pi},\bm{\sigma}}_{i,+}
\end{equation}
where the set of paths $\Path^{\bm{\pi},\bm{\sigma}}(l,i+)$ is as
per \Cref{Path space}.
  \end{proposition}

  \begin{proof}

A policy for \textit{Max} is a 
map 
$\bm{\pi}:S^+ \to \{0,1\}$ and 
a policy for \textit{Min} is a map 
$\bm{\sigma}:S^- \to \{0,1\}$.

So a pair of policies is any assignment of $0$ or $1$ to the game states (the cemetery state $\bot$ is trivially assigned $0$).
Given a neuron $(l,i)$ we want to compute the value $V^l_{i,+}$. 

We need to consider all game trajectories starting 
at state $(l,i,+)$ and continuing through states labeled $1$, until they reach a state labeled $0$. 
This corresponds to a unique trajectory of neurons of the net. 
We now need to assign a probability to such a trajectory as well as a total reward acummulated along the trajectory.

Consider a pair of policies fixed
and consider a game trajectory $\alpha \in \Path^{\bm{\pi},\bm{\sigma}}_{(l,i+)}$. We denote by $\alpha(t)$ the game state of the trajectory $\alpha$ at time $t$. Therefore we have 
$\alpha(l)=(l,i+)$ and
\begin{equation}
\alpha\coloneqq\left[\alpha(l), \alpha(l+1),
\alpha(l+2)
\dots
\alpha(l+k)\right]
\end{equation}

The probability of $\alpha$ is 
\begin{equation}
P(\alpha)\coloneqq P^l_{\alpha(l),\alpha(l+1)}
P^{l+1}_{\alpha(l+1),\alpha(l+2)} \dots
P^{l+k-1}_{\alpha(l+k-1),\alpha(l+k)}\enspace .
\end{equation}
The reward along $\alpha$ is the sum of discounted  biases over max states minus the sum of discounted biases over min states for $\nu\coloneqq 0\dots k$, plus or minus the terminal reward:

\begin{equation}
r(\alpha)\coloneqq
\sum_{\nu=0}^{\len(\alpha)-1}
\sgn((\alpha(l+\nu))\left(\prod_{\mu=l}^{l+\nu-1}  \gamma^\mu\right)
b^{l+\nu}_{\alpha(l+\nu)} 
+\delta_{\len(\alpha),L}
\left(\prod_{\mu=l}^{L-1}  \gamma^\mu\right)
\Phi_{\alpha(L)}(x)
\label{eq-ralpha}
\end{equation}
Then we have
\begin{equation}
V^{l,\bm{\pi},\bm{\sigma}}_{i,+}(x)\coloneqq
\sum_{\alpha \in \Path^{\bm{\pi},\bm{\sigma}}_{(l,i+)}} P(\alpha)r(\alpha)
\end{equation}
Finally we have
\begin{equation}
  y^l_i=V^l_{i+}(x)=\max_{\bm{\pi}}\min_{\bm{\sigma}}V^{l,\bm{\pi},\bm{\sigma}}_{i,+}(x)\label{e-path1}
\end{equation}

\end{proof}
Given an input $x=(x_1,\dots, x_{n_L})$ to a ReLU net, $x$ becomes the terminal reward of the corresponding ReLU net game. The value of the game is realized by optimal policy sequences
for the two players.
To make the dependence on $x$ explicit we denote these optimal policy sequences by 
$\bm{\pi}^*(x)$ and $\bm{\sigma}^*(x)$.
We then have:
\begin{proposition}
\label{optimal pol}
Let  $x=(x_1,\dots, x_{n_L})$, be the input to a ReLU neural net. The optimal policies $\bm{\pi}^*(x)$ and $\bm{\sigma}^*(x)$ for the corresponding  ReLU net game satisfy
\begin{equation}
\label{GameOP1}
 \bm{\pi}^*(x)((l,i,+))=
\begin{cases}
  1, & \text{if } \sum_j W^l_{i,j}y^{l+1}_j +b^l_i > 0, \\
  0  & \text{if }\sum_j W^l_{i,j}y^{l+1}_j+b^l_i < 0 
\end{cases}
\end{equation}
and\todo{SG: put strict inequalities, for weak there is a tie}
\begin{equation}
\label{GameOP2}
\bm{\sigma}^*(x)((l,i,-))=
1- \bm{\pi}^*(x)(l,i+)\enspace .
\end{equation}
When $\sum_j W^l_{i,j}y^{l+1}_j +b^l_i =0$, the choice of the ``stop'' or ``continue'' actions
is indifferent.
\end{proposition}
\begin{proof}

We saw in  \eqref{Game opt policy 1} and \eqref{Game opt policy 2}, that in a turn based game we can find the   optimal policies that realize the value of the game.
In our game they are as follows: for the \textit{Max player}
\begin{equation}
\label{GameOP1-bis}
 \bm{\pi}^*(x)((l,i,+))=
\begin{cases}
  1, & \text{if } \gamma^l_i [\sum_{W_{i,j} \geq 0} P^l_{i+,j+} V^{l+1}_{j+}(x) +
        \sum_{W_{i,j} \leq 0} P^l_{i+,j-} V^{l+1}_{j-}(x)]
        +b^l_i > 0, \\
  0  & \text{if }\gamma^l_i [\sum_{W_{i,j} \geq 0} P^l_{i+,j+} V^{l+1}_{j+}(x) +
        \sum_{W_{i,j} \leq 0} P^l_{i+,j-} V^{l+1}_{j-}(x)]
        +b^l_i < 0 .
\end{cases}
\end{equation}
and for the \textit{Min} player 
\begin{equation}
\label{GameOP2-bis}
 \bm{\sigma}^*(x)((l,i,-))=
\begin{cases}
  1, & \text{if } \gamma^l_i [\sum_{W_{i,j} \geq 0} P^l_{i+,j+} V^{l+1}_{j+} (x)+
        \sum_{W_{i,j} \leq 0} P^l_{i+,j-} V^{l+1}_{j-}(x)]
        +b^l_i < 0, \\
  0  & \text{if }\gamma^l_i [\sum_{W_{i,j} \geq 0} P^l_{i+,j+} V^{l+1}_{j+}(x) +
        \sum_{W_{i,j} \leq 0} P^l_{i+,j-} V^{l+1}_{j-}(x)]
        +b^l_i > 0 .
\end{cases}
\end{equation}
Or as we saw in the proof of Theorem \ref{theorem 1} we can simplify this to 
\begin{equation}
\label{GameOP1-ter}
 \bm{\pi}^*(x)((l,i,+))=
\begin{cases}
  1, & \text{if } \sum_j W^l_{i,j}y^{l+1}_j +b^l_i < 0, \\
  0  & \text{if }\sum_j W^l_{i,j}y^{l+1}_j+b^l_i < 0 .
\end{cases}
\end{equation}
and similarly for $\bm{\sigma}^*(x)((l,i,-))$.
\end{proof}
We see that given the input to the ReLU net we get these two Boolean vectors $\bm{\pi}^*(x)$  and $\bm{\sigma}^*(x)$ associated to it, which are exactly the optimal policies for the \textit{Max} and \textit{Min} player of the corresponding ReLU net game. 

\begin{proposition}
 Let $x\coloneqq(x_1,\dots x_{n_L})$ be the input to a ReLU net. Consider $x$ as the final reward for the corresponding ReLU net game. 
Let $\bm{\pi}^*(x)$ and $\bm{\sigma}^*(x)$ be the optimal policies for the corresponding ReLU net game. Let
$\Path^{\bm{\pi}^*(x),\bm{\sigma}^*(x)}_{(1,i+)}$ be the set of paths starting at layer 1, state $(1,i+)$  of the game (last layer of the ReLU net) and proceeding according to the optimal policies, then 
\begin{equation}
\label{Vop}
y^1_i=V^1_{i+}(x)=
\sum_{\alpha \in \Path^{\bm{\pi}^*(x),\bm{\sigma}^*(x)}_{(1,i+)}} r(\alpha)
\prod_{\nu=0}^{\len(\alpha)-1} 
P^{1+\nu}_{\alpha(1+\nu),\alpha(1+\nu+1)} 
\end{equation}
where
\begin{equation}
r(\alpha)\coloneqq\sum_{\nu=0}^{\len(\alpha)-1}
\sgn((\alpha(l+\nu))\left(\prod_{\mu=1}^{l+\nu-1}  \gamma^\mu\right)
b^{l+\nu}_{\alpha(l+\nu)} 
+\delta_{\len(\alpha),L}
\left(\prod_{\mu=1}^{L-1}  \gamma^\mu\right)
\Phi_{\alpha(L)}(x)
\end{equation}
\end{proposition}
\begin{proof}
  Specialize~\eqref{eq-ralpha} and~\eqref{e-path1} observing that the maximin
  in~\eqref{e-path1} is attained
  by the optimal policies $\bm{\pi}^*(x)$ and $\bm{\sigma}^*(x)$.\todo{SG: brevity}
\end{proof}

Figure~\ref{fig:game-paths}
depicts some possible game trajectories for a given input $x$ and corresponding optimal policies.

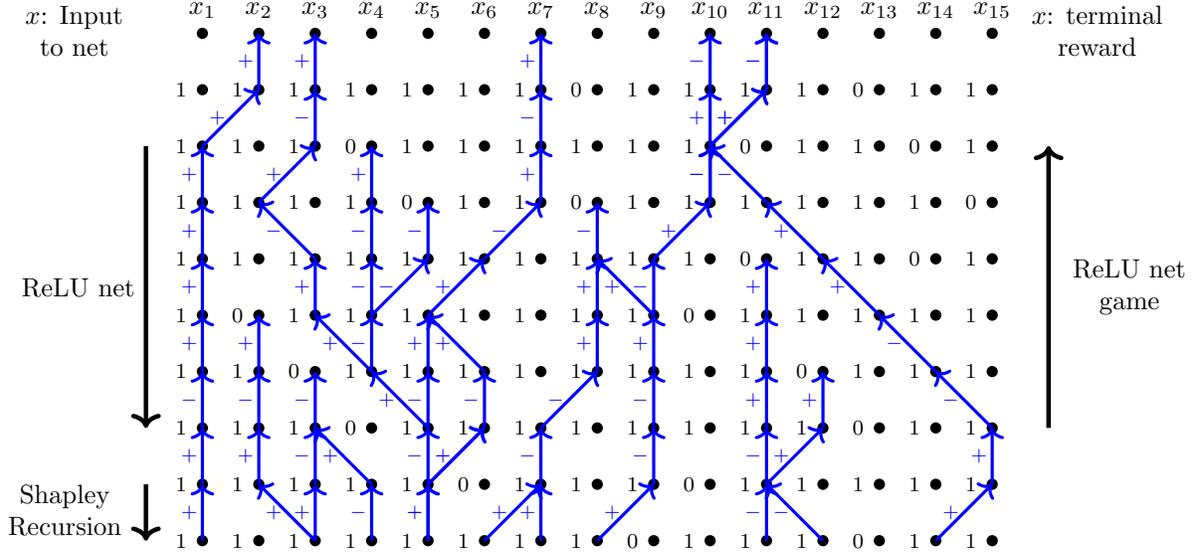
\begin{figure}[ht]
  \begin{center}
\begin{tikzpicture}[scale=0.75]

\foreach \x/\v in {
    0/1,1/1,2/0,3/1,4/1,5/1,6/1,7/1,8/1,9/0,10/0,11/1,12/1,13/1,14/1
} {
    \pgfmathtruncatemacro{\idx}{\x + 1} %
    \node[circle, fill=black, inner sep=1.5pt,
          label=above:$x_{\idx}$] at (\x,9) {};
}

        \foreach \x/\v in {0/1,1/1,2/1,3/1,4/1,5/1,6/1,7/0,8/1,9/1,10/1,11/1,12/0,13/1,14/1}
            \node[circle, fill=black, inner sep=1.5pt, label=left:{\footnotesize\v}] at (\x,8) {};

        \foreach \x/\v in {0/1,1/1,2/1,3/0,4/1,5/1,6/1,7/1,8/1,9/1,10/0,11/1,12/1,13/0,14/1}
            \node[circle, fill=black, inner sep=1.5pt, label=left:{\footnotesize \v}] at (\x,7) {};

        \foreach \x/\v in {0/1,1/1,2/1,3/1,4/0,5/1,6/1,7/0,8/1,9/1,10/1,11/1,12/1,13/1,14/0}
            \node[circle, fill=black, inner sep=1.5pt, label=left:{\footnotesize \v}] at (\x,6) {};

        \foreach \x/\v in {0/1,1/1,2/1,3/1,4/1,5/1,6/1,7/1,8/1,9/1,10/0,11/1,12/1,13/0,14/1}
            \node[circle, fill=black, inner sep=1.5pt, label=left:{\footnotesize \v}] at (\x,5) {};

        \foreach \x/\v in {0/1,1/0,2/1,3/1,4/1,5/1,6/1,7/1,8/1,9/0,10/1,11/1,12/1,13/1,14/1}
            \node[circle, fill=black, inner sep=1.5pt, label=left:{\footnotesize \v}] at (\x,4) {};

        \foreach \x/\v in {0/1,1/1,2/0,3/1,4/1,5/1,6/1,7/1,8/1,9/1,10/1,11/0,12/1,13/1,14/1}
            \node[circle, fill=black, inner sep=1.5pt, label=left:{\footnotesize \v}] at (\x,3) {};

        \foreach \x/\v in {0/1,1/1,2/1,3/0,4/1,5/1,6/1,7/1,8/1,9/1,10/1,11/1,12/0,13/1,14/1}
            \node[circle, fill=black, inner sep=1.5pt, label=left:{\footnotesize \v}] at (\x,2) {};

        \foreach \x/\v in {0/1,1/1,2/1,3/1,4/1,5/0,6/1,7/1,8/1,9/0,10/1,11/1,12/1,13/1,14/1}
            \node[circle, fill=black, inner sep=1.5pt, label=left:{\footnotesize \v}] at (\x,1) {};

        \foreach \x/\v in {0/1,1/1,2/1,3/1,4/1,5/1,6/1,7/1,8/0,9/1,10/1,11/1,12/0,13/1,14/1}
            \node[circle, fill=black, inner sep=1.5pt, label=left:{\footnotesize \v}] at (\x,0) {};

        \draw[very thick, blue, ->] (0,0) -- node[midway,left=-2pt] {{\scriptsize$+$}} (0,1);
        \draw[very thick, blue, ->] (0,1) -- node[midway,left=-2pt] {{\scriptsize$+$}} (0,2);
        \draw[very thick, blue, ->] (0,2) -- node[midway,left=-2pt] {\scriptsize$-$} (0,3);
        \draw[very thick, blue, ->] (0,3) -- node[midway,left=-2pt] {{\scriptsize$+$}} (0,4);

        \draw[very thick, blue, ->] (0,4) -- node[midway,left=-2pt] {{\scriptsize$+$}} (0,5);

        \draw[very thick, blue, ->] (0,5) -- node[midway,left=-2pt] {{\scriptsize$+$}} (0,6);

        \draw[very thick, blue, ->] (0,6) -- node[midway,left=-2pt] {{\scriptsize$+$}} (0,7);

        \draw[very thick, blue, ->] (0,7) -- node[midway,left=-2pt] {{\scriptsize$+$}} (1,8);

        \draw[very thick, blue, ->] (1,8) -- node[midway,left=-2pt] {{\scriptsize$+$}} (1,9);

\draw[very thick, blue, ->] (1,3) -- node[midway,left=-2pt] {{\scriptsize$+$}} (1,4);
        \draw[very thick, blue, ->] (2,0) -- node[midway,left=-2pt] {{\scriptsize$+$}} (1,1);
        \draw[very thick, blue, ->] (1,1) -- node[midway,left=-2pt] {{\scriptsize$+$}} (1,2);
        \draw[very thick, blue, ->] (1,2) -- node[midway,left=-2pt] {\scriptsize$-$} (1,3);

        \draw[very thick, blue, ->] (3,0) -- node[midway,left=-2pt] {\scriptsize$-$} (3,1);
        \draw[very thick, blue, ->] (3,1) -- node[midway,left=-2pt] {{\scriptsize$+$}} (2,2);
        \draw[very thick, blue, ->] (3,3) -- node[midway,left=-2pt] {\scriptsize$-$} (3,4);
        \draw[very thick, blue, ->] (3,4) -- node[midway,left=-2pt] {\scriptsize$-$} (4,5);
        \draw[very thick, blue, ->] (3,6) -- node[midway,left=-2pt] {{\scriptsize$+$}} (3,7);

        \draw[very thick, blue, ->] (4,0) -- node[midway,left=-2pt] {{\scriptsize$+$}} (4,1);
        \draw[very thick, blue, ->] (4,1) -- node[midway,left=-2pt] {\scriptsize$-$} (5,2);
        \draw[very thick, blue, ->] (5,2) -- node[midway,left=-2pt] {\scriptsize$-$} (5,3);
        \draw[very thick, blue, ->] (5,3) -- node[midway,left=-2pt] {{\scriptsize$+$}} (4,4);
        \draw[very thick, blue, ->] (4,5) -- node[midway,left=-2pt] {\scriptsize$-$} (4,6);

        \draw[very thick, blue, ->] (6,0) -- node[midway,left=-2pt] {{\scriptsize$+$}} (6,1);
        \draw[very thick, blue, ->] (6,1) -- node[midway,left=-2pt] {\scriptsize$-$} (6,2);
        \draw[very thick, blue, ->] (6,2) -- node[midway,left=-2pt] {\scriptsize$-$} (7,3);
        \draw[very thick, blue, ->] (7,3) -- node[midway,left=-2pt] {{\scriptsize$+$}} (7,4);
        \draw[very thick, blue, ->] (8,4) -- node[midway,left=-2pt] {{\scriptsize$+$}} (7,5);
        \draw[very thick, blue, ->] (7,5) -- node[midway,left=-2pt] {\scriptsize$-$} (7,6);

        \draw[very thick, blue, ->] (7,0) -- node[midway,left=-2pt] {{\scriptsize$+$}} (8,1);
        \draw[very thick, blue, ->] (8,1) -- node[midway,left=-2pt] {\scriptsize$-$} (8,2);
        \draw[very thick, blue, ->] (8,2) -- node[midway,left=-2pt] {\scriptsize$-$} (8,3);
        \draw[very thick, blue, ->] (10,3) -- node[midway,left=-2pt] {{\scriptsize$+$}} (10,4);

\draw[very thick, blue, ->] (10,4) -- node[midway,left=-2pt] {{\scriptsize$+$}} (10,5);
        \draw[very thick, blue, ->] (11,0) -- node[midway,left=-2pt] {\scriptsize$-$} (10,1);
        \draw[very thick, blue, ->] (10,1) -- node[midway,left=-2pt] {{\scriptsize$+$}} (11,2);
        \draw[very thick, blue, ->] (11,2) -- node[midway,left=-2pt] {{\scriptsize$+$}} (11,3);

        \draw[very thick, blue, ->] (5,0) -- node[midway,left=-2pt] {{\scriptsize$+$}} (6,1);
        \draw[very thick, blue, ->] (4,1) -- node[midway,left=-2pt] {\scriptsize$-$} (4,2);
        \draw[very thick, blue, ->] (4,2) -- node[midway,left=-2pt] {\scriptsize$-$} (4,3);
        \draw[very thick, blue, ->] (4,3) -- node[midway,left=-2pt] {{\scriptsize$+$}} (4,4);
        \draw[very thick, blue, ->] (2,4) -- node[midway,left=-2pt] {{\scriptsize$+$}} (2,5);
        \draw[very thick, blue, ->] (2,5) -- node[midway,left=-2pt] {\scriptsize$-$} (1,6);
        \draw[very thick, blue, ->] (1,6) -- node[midway,left=-2pt] {{\scriptsize$+$}} (2,7);
        \draw[very thick, blue, ->] (2,7) -- node[midway,left=-2pt] {\scriptsize$-$} (2,8);
        \draw[very thick, blue, ->] (2,8) -- node[midway,left=-2pt] {{\scriptsize$+$}} (2,9);

\draw[very thick, blue, ->] (4,2) -- node[midway,left=-2pt] {{\scriptsize$+$}} (3,3);

\draw[very thick, blue, ->] (3,3) -- node[midway,left=-2pt] {{\scriptsize$+$}} (2,4);

        \draw[very thick, blue, ->] (10,0) -- node[midway,left=-2pt] {\scriptsize$-$} (10,1);
        \draw[very thick, blue, ->] (10,1) -- node[midway,left=-2pt] {\scriptsize$-$} (10,2);
        \draw[very thick, blue, ->] (10,2) -- node[midway,left=-2pt] {{\scriptsize$+$}} (10,3);
        \draw[very thick, blue, ->] (8,3) -- node[midway,left=-2pt] {{\scriptsize$+$}} (8,4);
        \draw[very thick, blue, ->] (8,4) -- node[midway,left=-2pt] {\scriptsize$-$} (8,5);
        \draw[very thick, blue, ->] (9,6) -- node[midway,left=-2pt] {\scriptsize$-$} (9,7);
        \draw[very thick, blue, ->] (9,7) -- node[midway,left=-2pt] {{\scriptsize$+$}} (9,8);
        \draw[very thick, blue, ->] (9,8) -- node[midway,left=-2pt] {\scriptsize$-$} (9,9);

        \draw[very thick, blue, ->] (13,0) -- node[midway,left=-2pt] {{\scriptsize$+$}} (14,1);
        \draw[very thick, blue, ->] (14,1) -- node[midway,left=-2pt] {{\scriptsize$+$}} (14,2);
        \draw[very thick, blue, ->] (14,2) -- node[midway,left=-2pt] {\scriptsize$-$} (13,3);
        \draw[very thick, blue, ->] (13,3) -- node[midway,left=-2pt] {\scriptsize$-$} (12,4);
        \draw[very thick, blue, ->] (12,4) -- node[midway,left=-2pt] {{\scriptsize$+$}} (11,5);
        \draw[very thick, blue, ->] (11,5) -- node[midway,left=-2pt] {{\scriptsize$+$}} (10,6);
        \draw[very thick, blue, ->] (10,6) -- node[midway,left=-2pt] {\scriptsize$-$} (9,7);
        \draw[very thick, blue, ->] (9,7) -- node[midway,left=-2pt] {{\scriptsize$+$}} (10,8);
        \draw[very thick, blue, ->] (10,8) -- node[midway,left=-2pt] {\scriptsize$-$} (10,9);

        \draw[very thick, blue, ->] (2,0) -- node[midway,left=-2pt] {{\scriptsize$+$}} (2,1);
        \draw[very thick, blue, ->] (2,1) -- node[midway,left=-2pt] {\scriptsize$-$} (2,2);
        \draw[very thick, blue, ->] (2,2) -- node[midway,left=-2pt] {\scriptsize$-$} (2,3);
        \draw[very thick, blue, ->] (4,4) -- node[midway,left=-2pt] {{\scriptsize$+$}} (5,5);
        \draw[very thick, blue, ->] (5,5) -- node[midway,left=-2pt] {\scriptsize$-$} (6,6);
        \draw[very thick, blue, ->] (6,6) -- node[midway,left=-2pt] {{\scriptsize$+$}} (6,7);
        \draw[very thick, blue, ->] (6,7) -- node[midway,left=-2pt] {\scriptsize$-$} (6,8);
        \draw[very thick, blue, ->] (6,8) -- node[midway,left=-2pt] {{\scriptsize$+$}} (6,9);

        \draw[very thick, blue, ->] (4,0) -- node[midway,left=-2pt] {\scriptsize$-$} (4,1);
        \draw[very thick, blue, ->] (4,1) -- node[midway,left=-2pt] {{\scriptsize$+$}} (5,2);
        \draw[very thick, blue, ->] (7,4) -- node[midway,left=-2pt] {{\scriptsize$+$}} (7,5);
        \draw[very thick, blue, ->] (8,5) -- node[midway,left=-2pt] {{\scriptsize$+$}} (9,6);
        \draw[very thick, blue, ->] (9,6) -- node[midway,left=-2pt] {\scriptsize$-$} (9,7);
        \draw[very thick, blue, ->] (9,7) -- node[midway,left=-2pt] {{\scriptsize$+$}} (10,8);
        \draw[very thick, blue, ->] (10,8) -- node[midway,left=-2pt] {\scriptsize$-$} (10,9);

\draw[very thick, blue, ->] (3,4) -- node[midway,left=-2pt] {\scriptsize$-$} (3,5);

        \draw[very thick, blue, ->] (3,5) -- node[midway,left=-2pt]{\scriptsize$-$} (3,6);

        \draw[thick, ->] (15,2) -- (15,7) node[midway,right]{\begin{tabular}{c}ReLU net\\ game\end{tabular}};

        \draw[thick, ->] (-1,7) -- (-1,2) node[midway,left]{ReLU net};

\draw[thick, ->] (-1,1) -- (-1,0) node[midway,left]{\begin{tabular}{c}Shapley\\ Recursion\end{tabular}};

\node[right] at (14.3,9) {\begin{tabular}{c} $x$: terminal\\ reward\end{tabular}};
\node[left] at (-1,9) {\begin{tabular}{c} $x$: Input\\ to net\end{tabular}};

\end{tikzpicture}
\end{center}
    \caption{The grid indicates the neurons in a ReLU net. 
    Each neuron $(l,i)$ corresponds to 2 game states, one $(l,i+)$  where \textit{Max} plays and one $(l,i-)$ where \textit{Min} plays, but we don't indicate this in the figure so as not to clutter it.
    A given input $x\coloneqq(x_1,\dots,x_{15})$ to the ReLU net is interpreted as the terminal reward of the ReLU net game. The corresponding   optimal policies $\bm{\pi}^*(x):S^+ \to \{0,1\}$ and $\bm{\sigma}^*(x): S^- \to \{0,1\}$ determine 2 Boolean patterns on the vertices (one for the \textit{Max} labeled states and one for the \textit{Min}, (which is exactly the opposite: see Prop. \ref{optimal pol}). 
    \textbf{Game paths} (drawn in blue) contributing to the value of the ReLU net game (which is equal to the output of the ReLU net) for the given  $x$, start at the bottom row and proceed through $1$-labeled vertices, ending either at a 0-labeled vertex before reaching the top or when they reach the top. A plus sign on an edge indicates that the corresponding weight is positive and therefore the same player keep playing. A minus sign indicates the corresponding weight is negative and therefore the player changes. The sum over paths \eqref{Vop} gives the value of the game which is equal to the output of the neural net.}
    \label{fig:game-paths}
\end{figure}

\begin{remark}
    Note from Eqn~\eqref{Vop} that once we fix the policy (in particular the optimal policy here)  the problem becomes linear. In fact it becomes a Markov chain with rewards. We can also think of it as a random walk where the particle moving stops at the first stopping state it encounters.
\end{remark}

\begin{remark}
The map computed by the ReLU net is piecewise-linear and the linear pieces correspond exactly to optimal policies of the corresponding  ReLU net game. Namely a break from a linear piece to another, means that the optimal policy changes as a function of terminal reward for the game (i.e.\ input to the neural net).
\end{remark}
\subsection{Bounds on ReLU net output given bounds on input}

The representation of the output of the neural network as the value of a game
  provides a lift of the original neural network by a map
  which is {\em order preserving} with respect to the coordinate-wise order.
  
  More precisely, consider $x,x'\in \R^{k_L}$, together
  with the terminal reward $\bar{\phi}(x,x')$ such that $\bar{\phi}_{L,i+}(x,x')\coloneqq x_i$
  and $\bar{\phi}_{L,i-}(x, x')\coloneqq x'_i$, so that
  $\phi_{L,i,\epsilon}(x)=\bar{\phi}_{L,i,\epsilon}(x,-x)$, where $\epsilon=\pm$.
  \textit{This defines a more general game, extending the one of Section~\ref{sec-turn} -- the only difference being that the terminal cost $\bar{\phi}$ is general
  whereas the cost of the ReLU net game if of the form $\phi_L=(x,-x)$.}
  
  Let $\bar V^l_{i,\epsilon}(x,x')$ be the value of this new game,
  defined as per
  \eqref{e-def-vpisigma},\eqref{e-def-value},
  replacing $\phi_{L,i_L}(x)$ by $\bar{\phi}_{L,i_L}(x,x')$,
  and observe that $\bar V^l_{i,\epsilon}(x,x')$ still satisfies
  the Shapley-Bellman equations~\eqref{e-Val plus},\eqref{e-Val minus}.
  Since all the coefficients $P^l_{i\epsilon,j\eta}$
  arising in~\eqref{e-Val plus},\eqref{e-Val minus}, with
$\epsilon,\eta=\pm$ are nonnegative,
  it follows that the map $(x,x')\mapsto \bar V^l(x,x')$,
  obtained by composing order preserving ``layers'', is order
  preserving. Moreover, the output $y_i^l$ is obtained by specializing
  $y_i^l=\bar V^l_{i,+}(x,-x)$. This order preserving representation
  allows one to propagate bounds from the input to the output
  of the neural network, as shown by the following result.
\begin{proposition}\label{prop-bound}
  If the input $x=(x_1,\dots, x_n)$ to a ReLU neural net, belongs (coordinate-wise)  to some interval $[\underline{x},\overline{x}]$,
  then, we have that the output $y_i^l \in [\bar V^l_i(\underline{x},-\overline{x}),
    \bar V^l_i(\overline{x},-\underline{x})]$.
    \end{proposition}
    \begin{proof}
     We have   $\underline{x}\leq x\leq \overline{x}$ therefore 
     $-\overline{x}\leq -x\leq -\underline{x}$. Since the Shapley map $\bar V$ is order preserving  it follows that 
     $\bar V^l_i(\underline{x},-\overline{x})\leq \bar{V}^l(x,-x) \leq
    \bar V^l_i(\overline{x},-\underline{x})].$
    But $y^l_i=\bar{V}^l(x,-x)$,
    thereby proving the result.
    \end{proof}

\begin{remark}[Link with abstract interpretation]
  Theorem~\ref{theorem 1} can be interpreted in the light of \textit{static analysis
  of a program by abstract interpretation}. The latter method is a way
  to derive automatically program invariants, asserting that the
  vector of variables of the program stays
  in a parameterized set with a specified
  structure (box, polyhedron, ellipsoid,\dots),
  we refer the reader to~\cite{CC:77,Min2017} for
  background. 
  
  It turns out that the present Shapley operator
  coincides with an abstract semantic map obtained
  by applying abstract interpretation using boxes as a parameterized family of sets~\cite[\S~4.5]{Min2017}, thinking of the Neural network as a computer program.
We note that an analogy between abstract interpretation
and game theory was noted in~\cite{adjegaubertgoubault10}, in which a policy-type iteration was developed to compute
polyhedral program invariants. 

Beyond this analogy, the present results
show an actual ``embedding'' in zero-sum games.
In practice, boxes are generally known to provide coarse invariants.
We leave it for further work to extend the present approach
to more expressive, relational, domains~\cite{SSM:05,goubault2},
\end{remark}

\subsection{Interpreting Policies as certificates }\label{sec-certificates}
We next interpret the policies in terms of certificates allowing
one to verify properties of the neural network.
Let us assume that there is a single output, so that the neural
network can be used for a binary classification problem.
Let us fix thresholds $\alpha >  \beta$ and declare
that an input $x$ is accepted if $f(x) \geq \alpha$, rejected
if $f(x)\leq \beta$, and unclassified otherwise.
To simplify the notation, for every pair of policies
$\bm{\pi},\bm{\sigma}$ of Max and Min,
we denote by $f^{\pi,\sigma}$ the map which associates to $x$ the
value $V^{1,\pi,\sigma}_{1+}(x)$. %
We set $f^\pi =\inf_\sigma f^{\pi\sigma}$
and ${}^\sigma\! f = \sup_\pi f^{\pi\sigma}$.
Observe that $f^\pi$ is
concave and piecewise-linear, as it is an infimum
of affine maps. Dually, the map  ${}^\sigma\! f$ is
convex and piecewise-linear.
\begin{proposition}[Polyhedral representation of the accepted/rejected sets]
  \label{prop-poly}
  Every policy $\pi$ of Max determines a polyhedron
  \[
  C^\pi_\alpha =\{x\in \R^n\mid f^\pi(x) \geq \alpha\} \enspace,
  \]
  and the set $\actions=\{x\mid f(x)\geq \alpha\}$ of accepted inputs satisfies
  \begin{align}
  \actions = \bigcup_\pi C^\pi_\alpha \enspace .\label{e-cover-a}
  \end{align}
Dually, every policy $\sigma$ of Min determines a polyhedron
  \[
  {}^\sigma\! C_\beta =\{x\in \R^n\mid {}^\sigma\!f(x) \leq \beta\} \enspace,
  \]
  and the set of rejected inputs $\mathcal{R}=\{x\mid f(x)\leq \beta\}$ satisfies
  \begin{align}
\mathcal{R} =   \bigcup_\sigma {}^\sigma\! C_\beta \enspace .
  \label{e-cover-r}
  \end{align}
  \end{proposition}
  
\begin{proof}
  Since $f^\pi$ is concave and piecewise-linear, the super-level set of $f^\pi$,
  $C^\pi_\alpha$ ,is a polyhedron. Dually, the sub-level set ${}^\sigma\! C_\beta$ of ${}^\sigma\! f$
  is also a polyhedron.

  Observe that the following selection properties are satisfied
\[
\forall x\in \R^n\mid \exists \sigma,\pi,\; f(x) = f^\pi(x) = {}^\sigma\! f(x)
\enspace.
\]
The covering properties~\eqref{e-cover-a},\eqref{e-cover-r}
follow readily from this property.
  \end{proof}

\begin{remark}
  The cell coverings formulae~\eqref{e-cover-a} underly a logical interpretation
  of the game: Player Max (the ``prosecutor'') wants to select a policy $\pi$
  to certify the input $x$ has a certain property $(x\in C^\pi_\alpha)$, whereas Player Min
  (the defense) wants to select
  a policy $\sigma$ to certify the opposite property $(x\in {}^\sigma\!C_\beta)$.
  \end{remark}

\section{Basics of entropically  regularized Markov decision processes}

We now want to extend 
our previous construction to softplus neural nets. 
Indeed  recall that the softplus function $g_\tau$ is given by $g_\tau(a)\coloneqq\tau \log(1+e^\frac{a}{\tau})$ where $\tau \geq 0$.

The ReLU function is approximated by the softplus function since we have $\lim_{\tau \to 0}g_\tau(a)=\max(a,0)=ReLU(a)$.
\subsection{Shannon Entropy and free energy}

In order to interpret softplus neural nets as games we first recall that log-sum-exp is the Legendre-Fenchel transform of the Shannon entropy.

Indeed let $\Delta^n$ denote the n-simplex and let $p \in \Delta^n$ be a probability distribution. Consider
its Shannon entropy
\begin{equation}
    H(p)\coloneqq-\sum_i p_i\log(p_i)
\end{equation}
Let $Q$ be any vector in $\mathbb{R}^n$. We then have  the Legendre-Fenchel transform
\begin{equation}
  \label{eq:legendre max}
   \tau \log(\sum_{i=1}^n e^{\frac{Q_i}{\tau}})=\max_{p\in \Delta^n}(\langle p,Q\rangle +\tau H(p))
\end{equation}

Recall that in statistical mechanics the left hand side of $\eqref{eq:legendre max}$ is minus the free energy of a statistical ensemble where energies are $E_i\coloneqq-Q_i$ and $\tau$ is the temperature. Moreover the function 
$\sum_{i=1}^n e^{\frac{Q_i}{\tau}}$ is the partition function.

Then as is known  from statistical mechanics,
the optimal distribution $p^*$ (the one which realizes the maximum)  is the Gibbs distribution.  Indeed we have
\begin{equation}
    \label{Gibbs}
    p^*_i=\frac{e^{-E_i}}{\sum_{i=1}^n e^{\frac{-E_i}{\tau}}}=\frac{e^{Q_i}}{\sum_{i=1}^n e^{\frac{Q_i}{\tau}}}
\end{equation}
For the distribution $p^*$ the Shannon entropy $H(p^*)$ is the Gibbs entropy.

Note that when we take the zero temperature limit $\tau \to 0$ we obtain 
\begin{equation}
\label{temp 0}
\lim_{\tau \to 0}\tau \log(\sum_{i=1}^n e^{\frac{Q_i}{\tau}})=\max_i(Q_i)
\end{equation}

This is because as $\tau \to 0$ the right hand side of \eqref{eq:legendre max} becomes the $\max$ of a linear function over the simplex, which is convex. Therefore it will attain  its max values at the extremal points of the simplex. 

Because we will work with a game where one player maximizes and the other minimizes the reward, we will also need that 
\begin{equation}
  \label{eq:legendre min}
   -\tau \log(\sum_{i=1}^n e^{\frac{-Q_i}{\tau}})=\min_{p\in \Delta^n}(\langle p,Q\rangle -\tau H(p))
\end{equation}
which follows from \eqref{eq:legendre max}. 

We start by recalling the MDP case.

\subsection{Entropically regularized MDP}

Let \(S\) be the state space, \(A\) the action space, \(P(s'\mid s,a)\) the transition kernel, \(r_t(s,a,s')\) the stage reward at time \(t\), and \(\phi(s)\) the terminal reward at time \(T\).  Introduce an entropy‐regularization parameter \(\tau>0\).

Recall \eqref{MDP back Kolm} that we have
for a given policy $\pi_t: S \to \Delta(A(s))$, that

 \begin{equation}
V_{s,\tau}^{t,\pi_t}
= \sum_{a_i\in\actions(s)} \pi_t(a_i|s)
  \Bigl[
    r(s,a_i,s')
    + \sum_{s'\in\states} \gamma^t P(s'\mid s,a_i)\,V_{s,\tau}^{t,\bm{\pi}}(s')
  \Bigr],
\end{equation}

with the terminal condition 
\[
V^T_{s,\tau} = \phi(s).
\]
We now define \textit{the entropic regularization of the MDP with parameter $\tau$, to be the addition of $-\tau \log$ of the probability of the distribution of actions, to the reward at a given state}. Therefore, according to \eqref{MDP back Kolm}, 
for a given policy $\bm{\pi}$ the Kolmogorov recursion is given by 
\begin{equation}
\label{KolmEntMDP}
V^{t,\bm{\pi}}_{s,\tau}
=
\sum_{a_i\in\actions(s)} \pi(a_i \mid s)
\Bigl[
r(s,a_i)-\tau
\log \pi(a_i \mid s)
+\sum_{s'\in\states} \gamma^t P(s' \mid s,a_i)\,V^{t+1,\bm{\pi}}_{s',\tau}  \Bigr].
\end{equation}
This regularization has appeared in several contexts, see~\cite{pmlr-v97-geist19a} for background.

The value $V^t_{s,\tau}$ is given by 
$$V^t_{s,\tau}\coloneqq\max_{\pi(\cdot|s)} V^{t,\bm{\pi}}_{s,\tau}.$$

Consequently we have the Bellman recursion
\begin{equation}
\label{MDPValent1}
V^t_{s,\tau}
=\max_{\pi(\cdot|s)}
\sum_{a_i\in\actions(s)} \pi(a_i \mid s)
\Bigl[
r(s,a_i)
-\tau\, \log \pi(a_i \mid s)
+\sum_{s'\in\states} \gamma^t P(s' \mid s,a_i)\,V^{t+1}_{s',\tau}
\Bigr].
\end{equation}

We now put 
$$p_i\coloneqq\pi(a_i|s),$$  
\begin{equation}
\label{Q}
Q^t_{\tau}(s,a_i)\coloneqq r(s,a_i)+\sum_{s'\in\states} \gamma^t P(s' \mid s,a_i)\,V^{t+1}_{s',\tau} \enspace.
\end{equation}
(When $\tau=0$, $Q^t_{\tau}$ is the usual $Q$-function arising in reinforcement learning~\cite{Bertsekas_Tsitsiklis_1996}.)
Recall that the Shannon entropy is:
$$H(\pi(\cdot|s))=-\sum_i\pi(a_i|s)\log(\pi(a_i|s),$$
We then have
\begin{equation}
V^t_{s,\tau}
=\max_{\pi(\cdot|s)}
\Bigl[\sum_{a_i\in\actions(s)} \pi(a_i \mid s)
Q^t_{\tau}(s,a_i)+\tau H(\pi(\cdot|s))\Bigr]
\end{equation}
Therefore from the  Legendre transform~\eqref{eq:legendre max}
it follows that 
 the optimal value is 
\begin{equation}
\label{EntValQ}
V^t_{s,\tau}
= \tau \log\Bigl(\sum_{a_i\in\actions(s)}
\exp\bigl(Q^t_{\tau}(s,a_i)/\tau\bigr)\Bigr).
\end{equation}

Moreover, as we already saw in \eqref{Gibbs},  
the Gibbs distribution is the optimal policy realizing the supremum in \eqref{MDPValent1}, namely
\begin{equation}
\label{Gibbs pol MDP}
\pi^*_{t,\tau}(a_i\mid s)
=\frac{\exp(\frac{Q^t_{\tau}(s,a_i)}{\tau})}
{\sum_{b\in\actions(s)}\exp(\frac{Q^t_{\tau}(s,b)}{\tau})}
=
\exp\!\Bigl(\tfrac{1}{\tau}\bigl[Q^t_{\tau}(s,a_i)-V_{t,\tau}(s)\bigr]\Bigr).
\end{equation}
\begin{remark}
We see that the value $V^t_{s,\tau}$ is the negative of the \textbf{free energy} for a statistical ensemble where the states are distributed according to the Gibbs distribution. Since the value is maximized, the free energy is minimized at the Gibbs distribution.  
\end{remark}

\section{Entropically regularized Turn‐Based  Games}
We now generalize the entropic MDP case to an entropically regularized turn based zero-sum game, by putting together Sections 6 and 3.2.
Recall that we partition the state‐space \(S = S^1 \cup S^2\cup\{\bot\}\), where
\[
S^1 = \{\text{states where Player Max plays}\}, 
\quad
S^2 = \{\text{states where Player Min plays}\},
\quad
\bot \text{ is a cemetery state.}
\]
Fix a finite horizon \(T\), discount$\gamma^t$ and entropic regularization parameter \(\tau>0\).  

Analogously to  the MDP case we define \textit{the entropic regularization of the turn based game to be the addition of $-\log$ of the probability of the distribution of actions, to  the reward of the Max player and the addition of $\log$ of the probability of the policy, to  the reward of the Min player}. 

Let \(V_t,\tau(s)\) denote the regularized value at stage \(t\) and state \(s\).
Let $\bm{\pi}$ denote the randomized policy for the Max player and $\bm{\sigma}$ denote the randomized policy for the Min player. 

Then, following the same logic as for the entropic MDP we have:

\paragraph{\textbf{Player Max’s Turn (\(s\in S^1\))}}

Set 
\begin{equation}
Q^{t,1}_{\tau}(s,a^1_i)
\coloneqq r_t(s,a^1_i)
  + \sum_{s'\in S^1 \cup S^2} \gamma^t P(s'\mid s,a^1_i)\,V^{t+1}_{s',\tau}.
\end{equation}
Then according to \eqref{Turnvalmax} and \eqref{EntValQ}, the value at a state where Player Max plays, is given by:
\begin{equation}
\label{TurnValeMax1}
V^t_{s,\tau}
= \max_{\pi(\cdot\mid s)\in\Delta(A^1(s))}
  \Biggl\{
    \sum_{a^1_i\in A^1(s)} \pi(a^1_i\mid s)\,Q^{t,1}_{\tau}(s,a^1_i)
    \;+\;\tau\,H\bigl(\pi(\cdot\mid s)\bigr)
  \Biggr\},
\end{equation}
where \(H(\pi)=-\sum_i\pi(a^1_i)\log\pi(a^1_i)\).  Equivalently we have
\begin{equation}
\label{TurnValeMax2}
V^t_s
= \tau \log\!\Biggl(
    \sum_{a^1_i\in A^1(s)}
    \exp\!\bigl(Q^{t,1}_{\tau}(s,a^1_i)/\tau\bigr)
  \Biggr),
\quad s\in S^1.
\end{equation}

\paragraph{\textbf{Player Min’s Turn (\(s\in S^2\))}}

Here we first note that we have \eqref{eq:legendre min}:
\begin{equation}
   -\tau \log \Bigl(\sum_{i=1}^n e^{\frac{-Q_i}{\tau}}\Bigr)=\min_{p\in \Delta^n}(\langle p,Q\rangle -\tau H(p))
\end{equation}

We then define
\begin{equation}
Q^{t,2}_{\tau}(s,a^2_i)
= r_t(s,a^2_i)
  + \sum_{s'\in S^1 \cup S^2} \gamma^t P(s'\mid s,a^2_i)\,V^{t+1}_{s',\tau}.
\end{equation}
The value at a state where Player Min plays is, according to \eqref{Turnvalmin}  and \eqref{EntValQ},  given by
\begin{equation}
\label{TurnValeMin1}
V^t_{s,\tau}
= \min_{\sigma(\cdot\mid s)\in\Delta(A^2(s))}
  \Biggl\{
    \sum_{a^2_i\in A^2(s)} \sigma(a^2_i\mid s)\,Q^{t,2}_{\tau}(s,a^2_i)
    \;-\;\tau\,H\bigl(\sigma(\cdot\mid s)\bigr)
  \Biggr\}.
\end{equation}
Equivalently,
\begin{equation}
\label{TurnValeMin2}
V^t_{s,\tau}
= -\,\tau \,\log\!\Biggl(
    \sum_{a^2\in A^2(s)}
    \exp\!\bigl(-\,Q^{t,2}_{\tau}(s,a^2)/\tau\bigr)
  \Biggr),
\quad s\in S^2.
\end{equation}

\paragraph{\textbf{Boundary Condition}}

At the terminal stage \(T\), set
\[
V^T_s=\phi(s).
\]

Finally we see that  the optimal policy  for the Max player is 
\begin{equation}
\label{Gibbs pol Max}
\pi^*_{t,\tau}(a_i\mid s)
=\frac{\exp(\frac{Q^t_{\tau}(s,a_i)}{\tau})}
{\sum_{b\in\actions(s)}\exp(\frac{Q^t_{\tau}(s,b)}{\tau})}
=
\exp\!\Bigl(\tfrac{1}{\tau}\bigl[Q^t_{\tau}(s,a)-V^t_{s,\tau}\bigr]\Bigr).
\end{equation}

While the optimal policy for the Min player is

\begin{equation}
\label{Gibbs pol Min}
\sigma^*_{t,\tau}(a_i\mid s)
=\frac{\exp(\frac{-Q^t_{\tau}(s,a_i)}{\tau})}
{\sum_{b\in\actions(s)}\exp(\frac{-Q^t_{\tau}(s,b)}{\tau})}
=
\exp\!\Bigl(\tfrac{1}{\tau}\bigl[-Q^t_{\tau}(s,a)+V^t_{s,\tau}\bigr]\Bigr).
\end{equation}
\section{Softplus neural net as a turn based, entropically regularized, stopping game}

We will now show that \textit{the output of a Softplus neural net is the same as the value of a two-player,  zero-sum, turn-based, stopping game which we call the Softplus net game}. This game will be the entropic regularizations of the ReLU net game.

\subsection{Reminder on Softplus neural nets}
Recall that the softplus function with temperature $\tau >0$, is 
$\phi_\tau(x)\coloneqq\tau\log(1+e^\frac{x}{\tau})$ where it is applied coordinate-wise when $x$ is vector.
Note that, \eqref{temp 0} $\lim_{\tau \to 0}\phi_\tau (x)=\max(x,0)= ReLU(x)$.

Assume as before that the network has $L$ layers and layer $l$ has $k_l$ neurons. 
Again, we number the layers starting from the output of the neural network (layer $1$) to its input (layer $L$).
The weight matrix of layer $l$ is denoted by $W^l$ and there are bias vectors $b^l \in \mathbb{R}^{k_l}$ in each layer $l$. 
The input vector is $x\in \mathbb{R}^{k_L}$.

Define the affine maps
$A_l(v)\coloneqq W^l(v)+b^l$.
Then the total output function of the net is $g_\tau(x):\mathbb{R}^{k_L} \to \mathbb{R}^{k_1}$ where 
\begin{equation}
\label{SoftNNo}
g_\tau(x)=\phi_\tau(A^1(\dots\phi_\tau(A^{L-1}(\phi_\tau(A^L(x))))\dots)).
\end{equation} 

As before, the reason for numbering the layers from the output layer to the input is so that time will move forward along the game. 
Recall from \eqref{NNo-bis} that $f(x)$, the output of the ReLU net with the same weights and biases. Clearly $\lim_{\tau \to 0}g_\tau(x)=f(x)$. 
\subsection{The Softplus net game}
To construct the Softplus net game we implement to our ReLU net game the entropic regularization for turn based games as explained in Section 7.

To that end we keep the states, transition probabilities and terminal reward as explained in Section 4. 
We modify only the state-action  rewards:

If Players $\textit{Max}$ and $\textit{Min}$ play according to policies $\bm{\pi}\coloneqq(\pi^1,\dots. \pi^L)$ and $\bm{\sigma}\coloneqq(\sigma^1,\dots , \sigma^L)$, respectively, 
with $\pi^l(\cont|i+)+\pi^l(\stop|i+)=1$
and 
$\sigma^l(\cont|i-)+\sigma^l(\stop|i-)=1$
we put for the entropically regularized rewards,
with two actions ``continue'' ($\cont$) and ``stop'' ($\stop$)
in every state,

\begin{align}
R^l(l,i+,\cont)\coloneqq b^l_i-\tau \log(\pi^l(\cont|i+))& \\
R^l(l,i+,\stop)\coloneqq-\tau \log(\pi^l(\stop|i+))& \\
R^l(l,i-,\cont)\coloneqq-b^l_i+\tau \log(\pi^l(\cont|i+))& \\
R^l(l,i-,\stop)\coloneqq\tau \log(\pi^l(\stop|i+))
\end{align}
Let $\eta$ denote the choice of $\cont$ or $\stop$ actions. 
Then the expected payoff received by Player Max in the game from time $l$ to time $L$, with
initial state $s$ of the form $(l,i\pm)$, is given by,  \eqref{e-def-vpisigma}
\eqref{game val pol}:
\begin{equation}
\label{vEpol}
V_{l,s,\tau}^{\bm{\pi},\bm{\sigma}}(x) = E^{\bm{\pi},\bm{\sigma}}\Big(R^l(s_l,\eta)+\gamma^l R^{l+1}(s_{l+1},\eta)+ \dots + (\prod_{k=l}^{L-2}\gamma^k) R^{L-1}(s_{l-1},\eta) + (\prod_{k=l}^{L-1} \gamma^k) \phi_{L,i_L}(x) |s_l=s\Big)
\end{equation}

Note that in the ReLU game
a stopping action has no reward while in the Softplus game a stopping action 
 with a certain policy probability has a reward given by the log of that  probability.

 Therefore we add in the notation for the reward explicitely the actions $\eta$.

 In the ReLU game 
the action does not appear explicitly in the  expectation value since the sequence of states a game trajectories goes through fully determines the actions and rewards.

Let us now consider the Shapley-Bellman recursion.

\paragraph{\textbf{Max Player}}
We first put 

\begin{equation}
Q^l_{i+,\tau}\coloneqq\gamma^l_i \Big[\sum_{W_{i,j} \geq 0} P^l_{i+,j+} V^{l+1}_{j+,\tau} +
        \sum_{W_{i,j} \leq 0} P^l_{i+,j-} V^{l+1}_{j-,\tau}\Bigr]
        +b^l_i
\end{equation}
Notice that 
$$\pi^l(\cont|i_+)Q^l_{i+,\tau}-
\tau \pi^l(\cont|i_+)\log(\pi^l(\cont|i_+))$$ is the expected reward for continuing, if the \textit{Max} player is at $(l,i+)$ and 
$$\pi^l(\stop|i_+)0 -\tau \pi^l(\stop|i_+)\log\pi^l(\stop|i_+)$$ is the expected reward for stopping. 

Now consider the Shannon entropy
\begin{equation}
    H(\pi^l_{i+})\coloneqq-\pi^l(\cont|i_+)\log(\pi^l(\cont|i_+)) -\pi^l(\stop|i_+)\log(\pi^l(\stop|i_+))
\end{equation}
Then we have the Shapley-Bellman equation 
\begin{equation}
        V^l_{i+,\tau}=
        \max_{\pi^l(\cdot|i+)\in \Delta^1}\Bigl(\pi^l(\cont|i_+)Q^l_{i+,\tau}+\pi^l(\stop|i+)0 +\tau  H(\pi^l_{i+})\Bigr)
\end{equation}
 Therefore from $\eqref{eq:legendre max}$ we get
   \begin{equation}
        V^l_{i+,\tau}=\tau \log(1+\exp(\frac{\gamma^l_i \Bigl[\sum_{W_{i,j} \geq 0} P^l_{i+,j+} V^{l+1}_{j+,\tau} +
        \sum_{W_{i,j} \leq 0} P^l_{i+,j-} V^{l+1}_{j-,\tau}\Bigr]
        +b^l_i}{\tau}))\label{e-sb-legendre}
    \end{equation}
    we can also write
  \begin{equation}
  \label{VEMax}
        V^l_{i+,\tau}=\tau \log(1+\exp(\frac{Q^l_{i+,\tau}}{\tau}))
    \end{equation}

The optimal policy for this player is given by the Gibbs policy:
\begin{equation}
\pi^l(\cont| i+)
=\frac{\exp \bigl(\tfrac{Q^l_{i+,\tau}}{\tau}\bigr)}
{1+\exp(\frac{Q^l_{i+,\tau}}{\tau})}
\end{equation}

\paragraph{\textbf{Min player}}
Analogously we put
\begin{equation}
Q^l_{i-,\tau}=\gamma^l_i \Bigl[\sum_{W_{i,j} \geq 0} P^l_{i-,j-} V^{l+1}_{j-,\tau} +
        \sum_{W_{i,j} \leq 0} P^l_{i-,j+} V^{l+1}_{j+,\tau}\Bigr]
        -b^l_i
\end{equation}
\begin{equation}
    H(\sigma^l_{i-})\coloneqq-\sigma^l(\cont|i-)\log(\sigma^l(\cont|i-)) -\sigma^l(\stop|i-)\log(\sigma^l(\stop|i-))
\end{equation}

Then
\begin{equation}
        V^l_{i-,\tau}=
        \min_{\sigma(\cdot|i-)\in \Delta^1}\Bigl(\sigma(\cont|i-)Q^l_{i-,\tau}+\sigma(\stop|i-)0 -\tau  H(\sigma^l_{i-})\Bigr)
\end{equation}

Therefore

 \begin{equation}
        V^l_{i-,\tau}=-\tau \log\Bigl(1+\exp(-\frac{\gamma^l_i [\sum_{W_{i,j} \geq 0} P^l_{i-,j-} V^{l+1}_{j-,\tau} +
        \sum_{W_{i,j} \leq 0} P^l_{i-,j+} V^{l+1}_{j+,\tau}]
        -b^l_i}{\tau})\Bigr)
      \label{e-sb-legendre-dual}
    \end{equation}

Or equivalently

 \begin{equation}
 \label{VEMin}
        V^l_{i-,\tau}=-\tau \log(1+\exp(-\frac{Q^l_{i-,\tau}}{\tau})).
    \end{equation}

The optimal policy for the Min  player is given by the Gibbs policy
\begin{equation}
\sigma^l(\cont| i-)
=\frac{\exp \bigl(-\frac{Q^l_{i-,\tau}}{\tau}\bigr)}
{1+\exp(-\frac{Q^l_{i-,\tau}}{\tau})}
\end{equation}

\begin{theorem}\label{theorem 2}
  The value of the $i$th-output of a softplus neural network of depth $L$, on input vector $x$, coincides with the value of the associated discounted turn-based, entropy regularized, stopping game in horizon $L$ with initial state $(i,+)$ and terminal payoff $\phi_{L,\cdot}(x)$.

  More precisely,  let $y^l_{i,\tau}$ be the output of the $i$th neuron in the $l$ layer of the neural network,
  so that the output of the neural net is $y^1_\tau$ and the input is $y^L_\tau$ where $L$ is the number of layers of the neural net; then
\begin{align}
y^l_{i,\tau}=  V^l_{i+,\tau}=-V^l_{i-,\tau} \enspace .
\end{align}
\end{theorem}
\begin{proof}
We prove this by induction. We have  trivially $V^L_{i+,\tau}(x)=x_i=y^L_i$ and $V^L_{i-,\tau}(x)=-x_i=-y^L_i$. 

Moreover, assume
$V^{l+1}_{j+,\tau}(x)=y^{l+1}(x)$ and $V^{l+1}_{j-,\tau}(x)=-y^{l+1}(x)$.
We have that
\begin{align*}
Q^l_{i+,\tau}& \coloneqq\gamma^l_i [\sum_{W_{i,j} \geq 0} P^l_{i+,j+} V^{l+1}_{j+,\tau} +
        \sum_{W_{i,j} \leq 0} P^l_{i+,j-} V^{l+1}_{j-,\tau}]
+b^l_i\\
&=\sum_j W^l_{i,j} y^{l+1}_j +b^l_i
\end{align*}
and it follows from~\eqref{VEMax} that $V^{l}_{i,+}=y^{l}_i$. 

The proof
that $V^{l}_{i,-}=-y^{l}_i$ is dual.

\end{proof}

Letting $\tau\to 0$ in the Shapley-Bellman equations~\eqref{e-sb-legendre},\eqref{e-sb-legendre-dual} of the entropy regularized game, we deduce
that the value of the ReLU net game is the limit of the value of the entropy regularized
game,  i.e., $\lim_{\tau\to 0}V^l_{i\pm,\tau}=V^l_{i\pm}$.
\bibliographystyle{alpha}
\bibliography{games}

@preamble{
   "\def\cprime{$'$} "
}

@ARTICLE{Maragos1,
  author={Maragos, Petros and Charisopoulos, Vasileios and Theodosis, Emmanouil},
  journal={Proceedings of the IEEE}, 
  title={Tropical Geometry and Machine Learning}, 
  year={2021},
  volume={109},
  number={5},
  pages={728-755},
  doi={10.1109/JPROC.2021.3065238}}

@misc{Maragos2,
      title={A Tropical Approach to Neural Networks with Piecewise Linear Activations}, 
      author={Vasileios Charisopoulos and Petros Maragos},
      year={2019},
      note={arXiv:1805.08749},
      archivePrefix={arXiv},
      primaryClass={stat.ML}
}

@misc{Maragos3,
      title={Tropical Polynomial Division and Neural Networks}, 
      author={Georgios Smyrnis and Petros Maragos},
      year={2019},
      note={arXiv:1911.12922},
      archivePrefix={arXiv},
      primaryClass={cs.LG}
}

@article{Kordonis2025,
  title = {Revisiting Tropical Polynomial Division: Theory,  Algorithms,  and Application to Neural Networks},
  volume = {36},
  ISSN = {2162-2388},
  url = {http://dx.doi.org/10.1109/TNNLS.2025.3570807},
  DOI = {10.1109/tnnls.2025.3570807},
  number = {9},
  journal = {IEEE Transactions on Neural Networks and Learning Systems},
  publisher = {Institute of Electrical and Electronics Engineers (IEEE)},
  author = {Kordonis,  Ioannis and Maragos,  Petros},
  year = {2025},
  month = sep,
  pages = {15978–15992}
}

@inproceedings{zhang2018,
  author       = {Liwen Zhang and
                  Gregory Naitzat and
                  Lek{-}Heng Lim},
  editor       = {Jennifer G. Dy and
                  Andreas Krause},
  title        = {Tropical Geometry of Deep Neural Networks},
  booktitle    = {Proceedings of the 35th International Conference on Machine Learning,
                  {ICML} 2018, Stockholmsm{\"{a}}ssan, Stockholm, Sweden, July
                  10-15, 2018},
  series       = {Proceedings of Machine Learning Research},
  volume       = {80},
  pages        = {5819--5827},
  publisher    = {{PMLR}},
  year         = {2018},
  url          = {http://proceedings.mlr.press/v80/zhang18i.html},
  bibsource    = {dblp computer science bibliography, https://dblp.org}
}

@book{sorin_repeated_games,
title = {Repeated games},
author = {Mertens, J.-F. and Sorin, S. and Zamir, S.},
fseries = {Econometric Society Monographs},
series = {Econom. Soc. Monogr.},
year = {2015},
volume = {55},
publisher = {Cambridge University Press},
address = {Cambridge},
mydoi = {10.1017/CBO9781139343275},
}

@article{shapley_stochastic,
title = {Stochastic games},
author = {Shapley, L. S.},
fjournal = {Proceedings of the National Academy of Sciences of the United States of America},
journal = {Proc. Natl. Acad. Sci. USA},
volume = {39},
number = {10},
year = {1953},
pages = {1095--1100},
mydoi = {10.1073/pnas.39.10.1095},
}

@article{ovchinnikov,
author = {Ovchinnikov, S.},
title = {Max-min representations of piecewise linear functions},
fjournal = {Beitr{\"a}ge zur Algebra und Geometrie},
journal = {Beitr. Algebra Geom.},
volume = {43},
number = {1},
year = {2002},
pages = {297--302},
url = {http://eudml.org/doc/225460},
}

@INCOLLECTION{kolokoltsov,
author={V. Kolokoltsov},
title={Linear additive and homogeneous operators},
booktitle={Idempotent analysis},
series={Adv. in Sov. Math.},
volume={13},
publisher={AMS},
address={RI},
year={1992},
}

@BOOK{maslovkololtsov95,
author={V. Kolokoltsov and V Maslov},
title={Idempotent analysis and applications},
publisher={Kluwer Acad. Publisher},
year={1997},
}

@book{Bertsekas_Tsitsiklis_1996,
  author    = {Dimitri P. Bertsekas and John N. Tsitsiklis},
  title     = {Neuro-Dynamic Programming},
  publisher = {Athena Scientific},
  address   = {Belmont, MA},
  year      = {1996},
  isbn      = {1-886529-10-8, 978-1-886529-10-6}
}

@InProceedings{pmlr-v97-geist19a,
  title     = {A Theory of Regularized Markov Decision Processes},
  author    = {Matthieu Geist and Bruno Scherrer and Olivier Pietquin},
  booktitle = {Proceedings of the 36th International Conference on Machine Learning},
  pages     = {2160--2169},
  year      = {2019},
  editor    = {Kamalika Chaudhuri and Ruslan Salakhutdinov},
  volume    = {97},
  series    = {Proceedings of Machine Learning Research},
  month     = {09--15 Jun},
  publisher = {PMLR},
  url       = {https://proceedings.mlr.press/v97/geist19a.html}
}

@INPROCEEDINGS{goubault2,
author={S. Gaubert and E. Goubault and A. Taly and S. Zennou},
title={Static Analysis by Policy Iteration in Relational Domains},
year={2007},
booktitle={Proceedings of the Proc. of the 16th European Symposium on Programming (ESOP'07)},
address={Braga (Portugal)},
date={24 March - 1 April, 2007},
publisher={Springer},
series={LNCS},
volume=4421,
pages={237--252},
doi={10.1007/978-3-540-71316-6\_17},
}

@Article{CC:77,
author = {P. Cousot and R. Cousot},
title = {Abstract Interpretation: A unified lattice model for
    static analysis of programs by construction of
    approximations of fixed points},
journal = {Principles of Programming Languages 4},
pages = {238--252},
year = {1977}   
}

@InProceedings{SSM:05,
 author = {S. Sankaranarayanan and H. Sipma and Z. Manna}, 
 title = {Scalable Analysis of Linear Systems using Mathematical Programming}, 
 booktitle = {VMCAI}, 
 volume = 3385,
 series = {LNCS}, 
 year = 2005
}

@incollection{adjegaubertgoubault10,
author={A. Adje and S. Gaubert and E. Goubault},
title={Coupling policy iteration with semi-definite relaxation to compute accurate numerical invariants in static analysis},
booktitle={Proceedings of the 19th European Symposium on Programming (ESOP 2010)},
series={Lecture Notes in Computer Science},
number={6012},
publisher={Springer},
pages={23--42},
year=2010,
doi={10.1007/978-3-642-11957-6_3},
}

@BOOK{whittle86,
author={P. Whittle},
title={Optimization over Time},
publisher={Wiley},
year={1986},
}

@book{puterman2014markov,
  author = {M. L. Puterman},
  publisher = {John Wiley \& Sons},
  title = {Markov decision processes: discrete stochastic dynamic programming},
  year = 2014
}

@book{solan,
author={E. Solan},
title={A Course in Stochastic Game Theory},
publisher={Cambridge University Press},
year={2022},
}

@article{Min2017,
  title = {Tutorial on Static Inference of Numeric Invariants by Abstract Interpretation},
  volume = {4},
  ISSN = {2325-1131},
  url = {http://dx.doi.org/10.1561/2500000034},
  DOI = {10.1561/2500000034},
  number = {3–4},
  journal = {Foundations and Trends in Programming Languages},
  publisher = {Emerald},
  author = {Miné,  Antoine},
  year = {2017},
  month = dec,
  pages = {120–372}
}

@article{Kohn2005,
  title = {A deterministic‐control‐based approach motion by curvature},
  volume = {59},
  ISSN = {1097-0312},
  url = {http://dx.doi.org/10.1002/cpa.20101},
  DOI = {10.1002/cpa.20101},
  number = {3},
  journal = {Communications on Pure and Applied Mathematics},
  publisher = {Wiley},
  author = {Kohn,  Robert and Serfaty,  Sylvia},
  year = {2005},
  month = aug,
  pages = {344–407}
}

@article{Peres2008,
  title = {Tug-of-war and the infinity Laplacian},
  volume = {22},
  ISSN = {1088-6834},
  url = {http://dx.doi.org/10.1090/S0894-0347-08-00606-1},
  DOI = {10.1090/s0894-0347-08-00606-1},
  number = {1},
  journal = {Journal of the American Mathematical Society},
  publisher = {American Mathematical Society (AMS)},
  author = {Peres,  Yuval and Schramm,  Oded and Sheffield,  Scott and Wilson,  David},
  year = {2008},
  month = jul,
  pages = {167–210}
}

@article{evans,
author={Evans, Lawrence C},
title={Some Min-Max Methods for the {H}amilton-{J}acobi Equation},
journal={Indiana University Mathematics Journal},
volume=33,
number=1,
year=1984,
pages={31--50},
}

@article{1605.04518,
Author = {Marianne Akian and St{\'e}phane Gaubert and Antoine Hochart},
Title = {Minimax representation of nonexpansive functions and application to zero-sum recursive games},
Year = {2018},
Eprint = {1605.04518},
journal={Journal of Convex Analysis},
number=1,
}

@article{montufar,
author = {Mont\'{u}far, Guido and Ren, Yue and Zhang, Leon},
title = {Sharp Bounds for the Number of Regions of Maxout Networks and Vertices of Minkowski Sums},
journal = {SIAM Journal on Applied Algebra and Geometry},
volume = {6},
number = {4},
pages = {618-649},
year = {2022},
doi = {10.1137/21M1413699},

URL = { 
    
        https://doi.org/10.1137/21M1413699
    
    

},
eprint = { 
    
        https://doi.org/10.1137/21M1413699
    
    

}
}

@inproceedings{levine,
author = {Levine, Sergey and Koltun, Vladlen},
title = {Continuous inverse optimal control with locally optimal examples},
year = {2012},
isbn = {9781450312851},
publisher = {Omnipress},
address = {Madison, WI, USA},
booktitle = {Proceedings of the 29th International Coference on International Conference on Machine Learning},
pages = {475–482},
numpages = {8},
location = {Edinburgh, Scotland},
series = {ICML'12}
}

@book{Norris1997MarkovChains,
  author    = {Norris, J. R.},
  title     = {Markov Chains},
  publisher = {Cambridge University Press},
  year      = {1997},
  series    = {Cambridge Series in Statistical and Probabilistic Mathematics},
  address   = {Cambridge},
  isbn      = {9780521633963}
}

@inproceedings{crown,
  author       = {Shiqi Wang and
                  Huan Zhang and
                  Kaidi Xu and
                  Xue Lin and
                  Suman Jana and
                  Cho{-}Jui Hsieh and
                  J. Zico Kolter},
  editor       = {Marc'Aurelio Ranzato and
                  Alina Beygelzimer and
                  Yann N. Dauphin and
                  Percy Liang and
                  Jennifer Wortman Vaughan},
  title        = {Beta-CROWN: Efficient Bound Propagation with Per-neuron Split Constraints
                  for Neural Network Robustness Verification},
  booktitle    = {Advances in Neural Information Processing Systems 34: Annual Conference
                  on Neural Information Processing Systems 2021, NeurIPS 2021, December
                  6-14, 2021, virtual},
  pages        = {29909--29921},
  year         = {2021},
  url          = {https://proceedings.neurips.cc/paper/2021/hash/fac7fead96dafceaf80c1daffeae82a4-Abstract.html},
  bibsource    = {dblp computer science bibliography, https://dblp.org}
}

@article{Huangsurvey2020,
  title = {A survey of safety and trustworthiness of deep neural networks: Verification,  testing,  adversarial attack and defence,  and interpretability},
  volume = {37},
  ISSN = {1574-0137},
  url = {http://dx.doi.org/10.1016/j.cosrev.2020.100270},
  DOI = {10.1016/j.cosrev.2020.100270},
  journal = {Computer Science Review},
  publisher = {Elsevier BV},
  author = {Huang,  Xiaowei and Kroening,  Daniel and Ruan,  Wenjie and Sharp,  James and Sun,  Youcheng and Thamo,  Emese and Wu,  Min and Yi,  Xinping},
  year = {2020},
  month = aug,
  pages = {100270}
}

@book{Aliprantis,
  author    = {Charalambos D. Aliprantis and Kim C. Border},
  title     = {Infinite Dimensional Analysis: A Hitchhiker's Guide},
  edition   = {3},
  publisher = {Springer Berlin Heidelberg},
  year      = {2006},
  isbn      = {978-3-540-29587-7},
  doi       = {10.1007/3-540-29587-9},
}

@book{Lasserre,
  author    = {On{\'e}simo Hern{\'a}ndez‐Lerma and Jean B. Lasserre},
  title     = {Discrete‐Time Markov Control Processes: Basic Optimality Criteria},
  series    = {Stochastic Modelling and Applied Probability},
  volume    = {30},
  publisher = {Springer New York},
  year      = {1996},
  isbn      = {0-387-94579-2},
  doi       = {10.1007/978-1-4612-0729-0},
}

\end{document}